\documentclass{article}

\usepackage{amssymb}
\usepackage{amsmath}
\newtheorem{theorem}{Theorem}
\newtheorem{proof}{Proof}
\usepackage[justification=centering]{caption}
\usepackage{multirow}
\usepackage{caption} 
\usepackage{times}
\usepackage{graphicx} % more modern
\usepackage{subfigure} 

\usepackage{natbib}

\usepackage{algorithm}
\usepackage{algorithmic}

\usepackage{hyperref}

\usepackage[accepted]{icml2017}

\icmltitlerunning{Black-Box Optimization in Machine Learning with Trust Region Based Derivative Free Algorithm}

\begin{document} 

\twocolumn[
\icmltitle{ Black-Box Optimization in Machine Learning \\ with Trust Region Based Derivative Free Algorithm}



% you can specify symbols, otherwise they are numbered in order
\icmlsetsymbol{equal}{*}

\begin{icmlauthorlist}
\icmlauthor{Hiva Ghanbari}{equal,to}
\icmlauthor{Katya Scheinberg}{equal,to}
\end{icmlauthorlist}

\icmlaffiliation{to}{Lehigh University, Bethlehem, PA, USA}

\icmlcorrespondingauthor{Hiva Ghanbari}{hiva.ghanbari@gmail.com}
\icmlcorrespondingauthor{Katya Scheinberg}{katyascheinberg@gmail.com}

\icmlkeywords{Black-Box Optimization, DFO, Bayesian Optimization, AUC, RBF-Kernel SVM}

\vskip 0.3in
]



\printAffiliationsAndNotice{}  % leave blank if no need to mention equal contribution
%\footnotetext{hi}

\begin{abstract} 
In this work, we utilize a Trust Region based Derivative Free Optimization (DFO-TR) method to directly maximize the Area Under Receiver Operating Characteristic Curve (AUC), which is a nonsmooth, noisy function. We show that AUC is a smooth function, in expectation, if the distributions of the positive and negative data points obey a jointly normal distribution. The practical performance of this algorithm  is compared to three prominent Bayesian optimization methods and random search. The presented numerical results show that DFO-TR surpasses Bayesian optimization and random search on various black-box optimization problem, such as maximizing AUC and hyperparameter tuning.
\end{abstract} 

%**********
% Sections
%**********
%Introduction
\section{Introduction}
Most machine  learning (ML) models rely on optimization tools to perform training. Typically these models are formed so that at least stochastic estimates of the gradient can be computed; for example, when optimizing least squares or logistic loss of a neural network on a given data set. Lately, however, with the increasing need to tune hyperparameters of  ML models, black-box optimization methods have been given significant consideration. These methods do not rely on any explicit gradient computation, but assume that only function values can be computed, usually with noise. 

There are two relatively independent directions of research for black-box optimization--Bayesian Optimization (BO) \cite{mockus,brochu}, predominantly popular in the ML community, and derivative free optimization (DFO) \cite{conn}--popular in the optimization community. There are other classes of methods for black-box optimization  developed in the fields of simulation optimization and engineering, but they are more specialized and we will not focus on them here. 

Both BO and DFO methods are usually applied to functions that are not known to be convex. The key difference between the BO and DFO methods, is that BO methods always contain a component that aims at the exploration of the space, hence seeking a global solution, while DFO methods are content with a local optimum. However, it has been shown in DFO literature \cite{JJMore} that DFO methods tend to escape shallow local minima and are quite well suited for problems with a few well defined local basins (and possibly many small local basins that appear due to noise). 

BO until recently have been established  as the method of choice for hyperparameter optimization (HPO). While BO methods have been shown to be effective at finding good solutions (not always globally optimal, as that can only be achieved in the limit), their efficiency slows down significantly as the number of iterations grows. Overall, the methods are quite computationally costly and scale poorly with the number of hyperparameters. Recently, the BO efficiency has been called into question in comparison with a simple random search \cite{LLi}, whose iterations require nothing, but function evaluations. Moreover, some improvements on random search have been proposed to incorporate cheeper function evaluations and  further increase its efficiency for HPO. 

In this paper, we will explore properties of an efficient class of DFO methods--model-based trust region methods--in application to problems in ML. We will show that these methods can be more efficient than BO and random search, especially for problems of  dimensions higher than 2 or 3. 
In the specific case of HPO,  hyperparameters can be continuous, discrete or categorical. While some DFO methods have been developed for  the case of optimization over categorical or binary variables, these methods  essentially rely on local search heuristics and we do not consider them here. 
Our goal is to examine, in detail, the behavior of various black-box methods in a purely continuous setting. We also aim to explore practical scalability of the methods with respect to the dimension of the search space and nonlinearity of the function. While we will list some experiments on HPO problems,
these problems are limited to three continuous hyperparameters. Hence, to perform our comparison on problems of larger dimension, we mainly focus on a different problem--optimizing Area Under Receiver Operating Characteristic (ROC) Curve (AUC) \cite{hanley}, over a set of linear classifiers.

AUC is a widely used measure for learning with imbalanced data sets, which are dominant in ML applications.
%The {ROC} curve was originally developed in the signal detection theory \cite{egan}, and over the recent years, has become popular in machine learning and data mining communities. The ROC curve represents a trade-off between the true positive and the false positive rate of a classifier and the AUC is defined as the area under this curve, which is a real number between $0$ and $1$. 
%
%Given a positive class $\mathcal{S}_+:=\{x_i^+:i = 1, \ldots, N_+ \}$, negative class $\mathcal{S}_-:=\{x_j^-:j = 1, \ldots, N_- \}$, and a  classifier $p(w,x)$, parametrized by $w\in \mathbb{R}^d$, the AUC value can be computed via {``Wilcoxon-Man-Whitney"} statistic, presented in \cite{mann}, as
%\begin{equation} \label{AUC0}
%F_{AUC}(f) = \frac{\sum_{i=1}^{N_+} \sum_{j=1}^{N_-} \mathbb{I}_p (w,x_i^+ , x_j^-)}{N_+  N_-},
%\end{equation}
%where $\mathbb{I}_f$ is an indicator function, so that
% \begin{equation} \label{step0}
% \mathbb{I}_p(w,x_i^+,x_j^-) = \begin{cases} +1, & \text {if} ~~ p(w,x_i ^+)> p(w,x_j^-),\\ 0, & \text {otherwise.} \end{cases}
%\end{equation}
Various results has been reported in terms of comparing the AUC value as a performance measure of a classifier versus the usual prediction accuracy, that is the total percentage of misclassified examples \cite{bradley,flach,scheffer,tang2}. In particular, in \cite{bradley}, AUC is used to evaluate the performance of some ML algorithms such as decision trees, neural networks, and some statistical methods, where, experimentally, it is shown that AUC has advantages over the accuracy. 
%In \cite{huang}, the probabilistic version of ``consistency" and ``discriminancy" of two measures in evaluating the performance of a classification algorithm have been studied. Based on this probabilistic setting, it is shown that AUC is consistent and more discriminating than the accuracy. 
%This result is only based on the case of balanced data sets, where each classifier is required to predict exactly half of the examples as the positive class and the other half as the negative class. 
In \cite{cortes}, a statistical analysis of the relationship between AUC and the error rate, including the expected value and the variance of AUC for a fixed error rate, has been presented. Their results show that the average AUC value monotonically increases with the classification accuracy, but in the case of uneven class distributions, the variance of AUC can be large. This observation implies that the classifiers with the same fixed low accuracy may have noticeably different AUC values. Therefore, optimizing AUC value directly may be desirable, however doing so using gradient-based optimization techniques is not feasible, because this function is a discontinuous step function, hence its gradients are either zero or undefined.

%This motivates the use of algorithms designed to directly optimize the AUC rather than algorithms which minimize the error rate.

%Based on some advantages of AUC as a performance measurement, few classification schemes have been designed to maximize AUC instead of overall accuracy. In \cite{flach}, a new splitting criteria for decision tree construction has been proposed which chooses the split with the highest local AUC value. In \cite{scheffer}, the concept of AUC Support Vector Machine (AUC SVM) has been proposed and the experiments with different types of kernels show that the AUC SVM produces higher AUC value than the classic SVM. In \cite{tang}, the Self-adaptive Differential Evolution with Neighborhood Search (SaNSDE) algorithm has been proposed to optimize the weights of neural networks with respect to AUC. Moreover, the empirical studies show that the proposed algorithm trains neural network with larger AUC than existing methods. 

%  The main challenge of directly minimizing the {ranking loss} is due to its discontinuity, resulting from the indicator function $\mathbb{I}_f$. 
%Thus, its gradient is either zero or is not defined  and consequently the gradient based optimization methods can not be applied. 
This difficulty motivates various state-of-the-art techniques optimizing an approximation of this discontinuous loss function. In \cite{mozer, raskutti,calders}, various smooth nonconvex approximation  of AUC has been optimized by the gradient descent method. Alternatively, a  {\em ranking loss}, which is defined as $1-$AUC value, is minimized  approximately, by replacing it with the  pairwise margin loss, such as exponential loss, logistic loss, and hinge loss \cite{Joachims,steck,rudin,jin}, which results in a convex problem. In terms of computational effort, each iteration of the gradient descent algorithm applied to the pairwise exponential or logistic loss has quadratic computational complexity with the number of training samples $N$. However, computing the gradient of the pairwise hinge loss can be done by a method with the reduced complexity of $\mathcal{O}(N \log N )$. The same method can be applied to compute the AUC value itself, which we utilize in our approach. 
In this work, we apply a variant of a model-based trust region derivative free method, called DFO-TR, \cite{conn} to directly maximize the AUC function over a set of linear classifiers, without using any hyperparameters. 
 We note that in HPO the black-box function is often the validation error or accuracy achieved by the classifier trained using some given set of hyperparameters. Hence, like AUC this function is often piecewise constant and discontinuous. Thus, we believe that optimizing AUC directly and HPO in continuous domains have many common properties as black-box optimization problems. The main goal of this work is to demonstrate the advantages of the DFO-TR framework over other black-box optimization algorithms, such as Bayesian optimization and random search for various ML applications. 
 
 Bayesian optimization is known in the ML community as a powerful tool for optimizing nonconvex objective functions, which are expensive to evaluate, and whose derivatives are not accessible.  In terms of required number of objective function evaluations, Bayesian optimization methods are considered to be some of the most efficient techniques \cite{mockus,Jones,brochu} for black-box problems of low effective dimensionality. 
 %Moreover, characteristics such as low effective dimensionality \cite{bergstra,chen,wang} and problem variants \cite{bardenet} are the properties of the interest in Bayesian optimization. 
In theory, Bayesian optimization methods seek global optimal solution, due to  their sampling schemes, which  trade-off between exploitation and exploration \cite{brochu,hutter}. Specifically, Bayesian optimization methods construct a probabilistic model by using point evaluations of the true function. Then, by using this model, the subsequent configurations of the parameters will be selected \cite{brochu} by optimizing an {\em acquisition function} derived from the model. The model is built based on all past evaluation points in an attempt to approximate the true function globally. As a result, the acquisition function is often not trivial to maintain and optimize and per iteration complexity of BO methods increases. On the other hand, 
DFO-TR and other model-based DFO methods content themselves with building a local model of the true function, hence maintenance of such models remains moderate and optimization step on each iteration is cheap. 

We compare DFO-TR with SMAC \cite{hoos}, SPEARMINT \cite{snoek}, and TPE \cite{bergstra2}, which are popular Bayesian optimization algorithms based on different types of model. We show that DFO-TR is capable of obtaining better or comparable objective function values using fewer function evaluations and a much better computational effort overall. We also show that DFO-TR is more efficient than random search, finding better objective function values faster.  We also discuss the convergence properties of DFO-TR and its stochastic variant \cite{matt}, and argue that these results apply to optimizing {\em expected } AUC value, when it is a smooth function. We suggest further improvements to the algorithm by applying stochastic function evaluations and thus reducing function evaluation complexity and demonstrate computational advantage of this approach.

In summary our contributions are as follows
\begin{itemize}
\item We provide a computational comparison that shows that model-based trust-region DFO methods can be superior to BO methods and random search on a variety of black-box problems over continuous variables arising in ML. 
\item We utilize recently developed theory of stochastic model-based trust region methods to provide theoretical foundations for applying the method to optimize AUC function over a set of linear classifiers. We show that this function is continuous in expectations under certain assumptions on the data set.
\item We provide a simple direct way to optimize AUC function, which is often more desirable than optimizing accuracy or classical loss functions. 
\end{itemize}

 This paper is organized in six sections. In the next section, we describe the practical framework of a {DFO-TR} algorithm. In \S\ref{first}, we describe how AUC function can be interpreted as a smooth function in expectation. In \S\ref{sec4}, we state the main similarities and differences between Bayesian optimization and DFO-TR. We present computational results in \S\ref{third}. Finally, we state our conclusions in \S\ref{fourth}.

%*********
% Section
%*********
\section{Algorithmic Framework of DFO-TR} \label{second}
Model-based trust region DFO methods \cite{ARConn_KScheinberg_PhLToint_1997b, MJDPowell_2004} have been proposed for a class of optimization problems of the form $\min_{w\in \mathbb{R}^d} f(w)$,  when computing the gradient and the Hessian of $f(w)$ is not possible, either because it is unknown or because the computable gradient is too noisy to be of use. It is, however, assumed that some local first-order or even second-order information of the objective function is possible to construct to an accuracy sufficient for optimization. If the function is smooth, then such information is usually constructed by building an interpolation or regression model of $f(w)$ using a set of points for which function value is (approximately known) \cite{conn}. By using quadratic models, these methods are capable of approximating the second-order information efficiently to speed up convergence and to guarantee convergence to local minima, rather than simply local stationary points. They have been shown to be the most practical black-box optimization methods in deterministic settings \cite{JJMore}. 
Extensive convergence analysis of these methods over smooth deterministic functions have been summarized in \cite{conn}.
%Algorithm \ref{al1} represents the algorithmic framework of {DFO-TR}, as a practical variant of the DFO methods, stated in \cite{conn}. While, the convergence analysis in \cite{conn} applies to smooth deterministic (not noisy) functions, Algorithm \ref{al1} does not have any convergence guarantees, when it applies to the following deterministic optimization problem
%\bequationn \label{deterministic}
%\min_{x} f(x).
%\eequationn
%However, it performs efficiently in practice \cite{ASBandeira}. On the other hand, r

Recently, several variants of trust region methods have been proposed to solve stochastic optimization problem $\min_w \mathbb{E_\xi}[f(w,\xi)]$, 
%\begin{equation*} \label{stochastic}
%\min_w \mathbb{E_\xi}[f(w,\xi)],
%\end{equation*}
where $f(w,\xi)$ is a stochastic function of a deterministic  vector $w\in \mathbb{R}^d$ and a random variable $\xi$ \cite{2015sarhasetal,Larson2015,matt}. In particular, in \cite{matt}, a trust region based stochastic method, referred to STORM (STochastic Optimization with Random Models), is introduced and shown to converge, almost surely, to a stationary point of $\mathbb{E_\xi}[f(w,\xi)]$, under the assumption that $\mathbb{E_\xi}[f(w,\xi)]$ is smooth. Moreover, in recent work \cite{BlanchetCartisMenickellyScheinberg2016} a convergence rate of this method has been analyzed. 
%However, if the constructed model is a fully quadratic model, then this algorithm converges to a local minimum not only an stationary point \cite{matt}.
This class of stochastic methods utilizes  samples of $f(w,\xi)$ to construct models that approximate  $\mathbb{E_\xi}[f(w,\xi)]$ sufficiently accurately, with high enough probability. 
In the next section, we will show that the AUC function can be a smooth function in expectation, under some conditions, hence STORM method and tis convergence properties are applicable.
For general convergent framework of STORM, we refer the reader to \cite{matt}. 

Here in Algorithm \ref{al1}, we present the specific practical implementation of a deterministic algorithm, which can work with finite training sets rather than infinite distributions, but shares many properties with STORM and produces very good results in practice. The key difference between STORM and DFO-TR is that the former requires resampling $f(w,\xi)$ for various $w$'s, at each iteration, since $f(w,\xi)$ is a random value for any fixed $w$, while DFO-TR computes only one value of deterministic $f(w)$ per iteration. When applied to deterministic smooth functions, this algorithm converges to a local solution \cite{conn},
but here we apply it to a nonsooth function which can be viewed as a noisy version of a smooth function (as argued in the next section). While there are no convergence results for DFO-TR or STORM for deterministic nonsooth, {\em noisy} functions, the existing results indicate that the DFO-TR method will converge to a neighborhood of the solution before the noise in the function estimates prevents further progress. Our computational results conform this. 

We note a few key properties on the algorithm. At each iteration, a quadratic model, not necessarily convex, is constructed using previously evaluated points that are sufficiently close to the current iterate. Then, this model is optimized  inside the trust region  $\mathcal{B}(w_k, \Delta_k) := \{w: \|w-w_k\| \leq \Delta_k\}$. The global solution for the trust region subproblem is well known and can be obtained efficiently in $O(d^3)$ operations \cite{TRBook},
which is not expensive, since in our setting $d$ is small. The number of points that are used to construct the model is at most $ \frac{(d+1)(d+2)}{2}$, but good models that exploit some second-order information can be constructed with $O(d)$ points. Each iteration requires only one new evaluation of the function and the new function value  either provides an improvement over the best observed value or can be used to improve the local model \cite{ScheToin09}. Thus the method utilizes function evaluations very efficiently. 

\begin{algorithm}[ht]
\caption{~\textbf{Trust Region based Derivative Free Optimization (DFO-TR)}}
\label{al1}
\begin{algorithmic}
 \STATE \textbf{1:} \textbf{Initializations:} 
\STATE \text{2:} Initialize $w_0$, $\Delta_0>0$, and choose $0<\eta_0<\eta_1<1$, \\
~~~~$\theta>1$, and $0< \gamma_1<1<\gamma_2$.
\STATE \text{3:} Define an interpolation set $\mathcal{W}\in \mathcal{B}(w_0,\Delta_0)$.  
\STATE \text{4:} Compute $f(w)$ for all $w\in \mathcal{W}$, let $m=|\mathcal{W}|$. %and store them in set $\mathcal{F}_{\mathcal X}$.
\STATE \text{5:} Let $w_0:=\bar{w_0}=\arg \min_{w \in \mathcal{W}} f(w)$. 
\STATE \textbf{6:} \textbf{for $k=1,2,\cdots$ do}
\STATE \text{7:}~~~~\textbf{Build the model:} 
\STATE \text{8:}~~~~Discard all $w\in \mathcal W$ such that $\|w-w_k\|\geq\theta \Delta_k$.
\STATE \text{9:}~~~~Using $\mathcal{W}$ construct an interpolation model: \\
~~~~~~~$Q_k(w)=f_k+g_k^T(w-w_k)$\\
~~~~~~~~~~~~~~~~~~~~~~~$+\frac{1}{2}(w-w_k)^TH_k(w-w_k)$.  
 \STATE \text{10:}~~~~\textbf{Minimize the model within the trust region:} 
 \STATE \text{11:}~~~~$\hat{w}_k= \arg \min_{w \in \mathcal{B}(w_k,\Delta_k)}Q_k(w)$.
  \STATE \text{12:}~~~~Compute  $f(\hat{w}_k)$ and  $\rho_k := \frac{f(w_k)-f(\hat{w}_k)}{ Q_k(w_k)-Q_k(\hat{w}_k)}$. 
\STATE \text{13:}~~~~\textbf{Update the interpolation set:} 
 \STATE \text{14:}~~~~\textbf{if} $m < \frac{1}{2}(d+1)(d+2)$ \textbf{then} 
 \STATE \text{15:}~~~~~~add new point $\hat{w}_k$ to the interpolation set $\mathcal W$, \\
 ~~~~~~~~~~~and $m:=m+1$.
 \STATE \text{16:}~~~~\textbf{else}
 \STATE \text{17:}~~~~~~\text{if} $\rho_k \geq \eta_0$ then replace $\arg\max_{w\in \mathcal{W}} \|w-w_k\|$ \\
 ~~~~~~~~~~ with $\hat{w}_k$,
 \STATE \text{18:}~~~~~~\text{otherwise} do the same if \\
 ~~~~~~~~~~ $\|w-w_k\|<\max _{w\in \mathcal{W}} \|w-w_k\|$. 
  \STATE \text{19:}~~~~\textbf{end if} 
% \STATE \text{18:}~~~~~~ Update set $\mathcal{F}_{\mathcal X}$ accordingly.
\STATE \text{20:}~~~~\textbf{Update the trust region radius:} 
\STATE \text{21:}~~~~if $\rho_k \geq \eta_1$ then $w_{k+1} \gets \hat{w}_k$ and $\Delta_{k+1}  \gets \gamma_2 \Delta_k$.
 \STATE \text{22:}~~~~if $\rho_k < \eta_0$ then $w_{k+1} \gets w_k$, and if $m > d+1$  \\
 ~~~~~~~~ update $\Delta_{k+1} \gets \gamma_1 \Delta_k$, otherwise $\Delta_{k+1} \gets \Delta_k$.
\STATE \textbf{23:} \textbf{end for.}
 \end {algorithmic}
\end{algorithm}

\section{AUC function and its expectation} \label{first}
In this section, we define the AUC function  of a linear classifier $w^Tx$ and demonstrate that under certain assumptions on the data set distribution, 
its expected value is smooth with respect to $w$. First, suppose that we have two given sets $\mathcal{S}_+:=\{x_i^+:i = 1, \ldots, N_+ \}$ and $\mathcal{S}_-:=\{x_j^-:j = 1, \ldots, N_- \}$, sampled from distributions $\mathcal{D}_1$ and  $\mathcal{D}_2$, respectively. For a given linear classifier $w^T x$, the corresponding AUC value, a (noisy) nonsmooth deterministic function, is obtained as \cite{mann}
\begin{equation} \label{AUC}
F_{AUC}(w) = \frac{\sum_{i=1}^{N_+} \sum_{j=1}^{N_-} \mathbb{I}_w (x_i^+ , x_j^-)}{N_+  N_-},
\end{equation}
where $\mathbb{I}_w$ is an indicator function defined as
 \begin{equation*} \label{step}
 \mathbb{I}_w(x_i^+,x_j^-) = \begin{cases} +1 & \text {if} ~~ w^Tx_i ^+> w^Tx_j^-,\\ 0 & \text {otherwise.} \end{cases}
\end{equation*}
%In this case, since we deal with a nonsmooth deterministic function $F_{AUC}(w)$, for the sake of performance of {DFO-TR}, one needs to guarantee that $F_{AUC}(w)$ is a smooth function in expectation \cite{matt}. Therefore, 
Clearly, $F_{AUC}(w)$ is piece-wise constant function of $w$, hence it is not continuous. 
However, we show that the expected value of AUC, denoted by $\mathbb{E}[F_{AUC}(w)]$, is a smooth function of the vector $w$, where the expectation 
is taken over $\mathcal{S}_+$ and $\mathcal{S}_-$, in some cases. 

To this end, first we need to interpret the expected value of AUC in terms of a probability value.
For two given finite sets ${\cal S_+}$ and ${\cal S_-}$, the AUC value of the classifier $w^Tx$ can be interpreted as
\begin{equation*} \label{noisy.auc}
F_{AUC}(w) = P\left(w^TX_+ > w^TX_-\right),
\end{equation*}
where $X_+$ and $X_-$ are randomly sampled from ${\cal S_+}$ and ${\cal S_-}$, respectively. Now, if two sets $\mathcal{S}_+$ and $\mathcal{S}_-$ are randomly drawn from distributions $\mathcal{D}_1$ and $\mathcal{D}_2$, then the expected value of AUC is defined as 
\begin{equation} \label{exp.auc}
\mathbb{E}\left[F_{AUC}(w)\right] = P\left(w^T\hat{X}_+ > w^T\hat{X}_-\right),
\end{equation}
where $\hat{X}_+$ and $\hat{X}_-$ are randomly chosen from distributions $\mathcal{D}_1$ and $\mathcal{D}_2$, respectively.
%\eremark
In what follows, we show that if two distributions $\mathcal{D}_1$ and  $\mathcal{D}_2$ are jointly normal, then $\mathbb{E}[F_{AUC}(w)]$ is a smooth function of $w$. We use the following results from statistic. 

\begin{theorem} \label{T1}
If  two $d-$dimensional random vectors $\hat{X}_1$ and $\hat{X}_2$ have a joint multivariate normal distribution, such that
\begin{equation} \label{p1}
\begin{pmatrix} \hat{X}_1 \\ \hat{X}_2 \end{pmatrix}
\sim \mathcal{N}\left( \mu, \Sigma \right),
\end{equation}
\begin{equation*}
 \text{where} ~~ \mu = \begin{pmatrix} \mu_1 \\ \mu_2 \end{pmatrix}~~\text{and}~~~ \Sigma = 
\begin{pmatrix} \Sigma_{11} & \Sigma_{12} \\ \Sigma_{21} & \Sigma_{22} \end{pmatrix}.
\end{equation*}
Then, the marginal distributions of $\hat{X}_1$ and $\hat{X}_2$ are normal distributions with the following properties
\begin{equation*}
\hat{X}_1 \sim \mathcal{N}\left(\mu_1, \Sigma_{11}\right), ~~~~~~~~~~\hat{X}_2 \sim \mathcal{N}\left(\mu_2, \Sigma_{22}\right).
\end{equation*}
\end{theorem}
\begin{proof}
The proof can be found in \cite{tong}.  $\square$
\end{proof}

%Using the result of Theorem \ref{T1}, the following result is  derived. 
\begin{theorem} \label{T2}
Consider two random vectors $\hat{X}_1$ and $\hat{X}_2$, as defined in \eqref{p1}, then for any vector $w \in \mathbb{R}^d$, we have
\begin{equation} \label{z}
Z = w^T \left(\hat{X}_1-\hat{X}_2\right) \sim \mathcal{N}\left(\mu_Z, \sigma_Z^2\right),
\end{equation}
\begin{equation} \label{zp}
\begin{aligned}
\text{where}~~ \mu_Z &= w^T\left(\mu_1 - \mu_2\right)\\
\text{and}~~\sigma_Z^2 &= w^T \left( \Sigma_{11} + \Sigma_{22} - \Sigma_{12} - \Sigma_{21} \right)w.
\end{aligned}
\end{equation}
\end{theorem}
\begin{proof}
The proof can be found in \cite{tong}.  $\square$
\end{proof} 

Now, in what follows, we have the formula for the expected value of the AUC.
\begin{theorem} \label{T3}
If two random vectors $\hat{X}_1$ and $\hat{X}_2$, respectively drawn from distributions $\mathcal{D}_1$ and $\mathcal{D}_2$, have a joint normal distribution as defined in Theorem \ref{T1}, then the expected value of AUC function can be defined as 
\begin{equation*}
\mathbb{E}\left[F_{AUC}(w)\right] = \phi\left(\frac{\mu_Z}{\sigma_Z}\right),
\end{equation*}
where $\phi$ is the cumulative function of the standard normal distribution, so that $\phi(x) = e^{-\frac{1}{2} x^2} / 2 \pi$, for $\forall x \in \mathbb{R}$.
\end{theorem} 
\begin{proof}
From Theorem \ref{T2} we have
\begin{equation*}
\begin{aligned}
& \mathbb{E}\left[F_{AUC}(w)\right] \\
& = P\left(w^T\hat{X}_+ > w^T\hat{X}_-\right) =  P\left(w^T(\hat{X}_+ - \hat{X}_-)>0\right)\\
& = P(Z > 0) = 1 - P\left(Z \leq 0\right) \\
&= 1 - P\left(\frac{Z - \mu_Z}{\sigma_Z} \leq \frac{ - \mu_Z}{\sigma_Z}\right)\\
& = 1 - \phi\left(\frac{ - \mu_Z}{\sigma_Z}\right)= \phi\left(\frac{\mu_Z}{\sigma_Z}\right),
\end{aligned}
\end{equation*}
where random variable $Z$ has been defined in \eqref{z}, with the stated mean and variance in \eqref{zp}. $\square$

\end{proof}

In Theorem \ref{T3}, since the cumulative function of the standard normal distribution, i.e., $\phi$, is a smooth function, we can conclude that for a given linear classifier $w^Tx$, the corresponding expected value of AUC, is a smooth function of $w$. Moreover, it is possible to compute derivatives of this function, if the first and second moments of the normal distribution are known. We believe that the assumption of the positive and negative data classes obeying jointly normal distribution is too strong to be satisfied for most of the practical problems,  hence we do not think that the gradient estimates of $\phi\left({\mu_Z}/{\sigma_Z}\right)$ will provide good estimates of the true expected AUC. However, we believe that  smoothness of the expected  AUC remains true in cases of 
many practical distributions.  The extension of this observation can be considered as a subject of a future study. 
 
%*********
% Section
%*********
\section{Bayesian Optimization versus DFO-TR}\label{sec4}
Bayesian optimization framework, as outlined in Algorithm~\ref{Bayes}, like DFO-TR framework
operates by constructing a (probabilistic) model $M(w)$ of the true function $f$ by using function values computed thus far by the algorithm.
The next iterate $w_k$ is computed by optimizing an {\em acquisition} function,  $a_M$, which presents a trade-off between minimizing the model and improving the model, by exploring areas where $f(w)$ has not been sampled.  Different Bayesian optimization algorithms use different models and different acquisition functions, for instance, {\em expected improvement} \cite{schonlau} over the best observed function value is a popular acquisition function in the literature. 

The key advantage and difficulty of BO methods is that the acquisition function may have a complex structure, and needs to be optimized globally on each iteration.  For example, the algorithm in \cite{brochu} uses  deterministic derivative free optimizer DIRECT \cite{DRJones} to maximize the acquisition function. When evaluation of $f(w)$ is very expensive, then the expense of optimizing the acquisition function may be small in comparison. However, in many case, as we will see in our computational experiments, this expense can be dominant. 
%based on using different model classes. The main strategy of Bayesian optimization is to use all of the information available from previous evaluations of $f$, in order to learn a global model. This procedure requires more computation to determine next candidate point to evaluate. However, it improves the efficiency of Bayesian optimization algorithms in terms of required number of function evaluations. Specially, when evaluation of $f$ is expensive, performing more computation to make a better decision could be an efficient strategy. On the contrary, 
In contrast, the DFO-TR method, as described in Algorithm \ref{al1}, maintains a quadratic model by using only the points in the neighborhood of 
the current iterate and global optimization of this model subject to the trust region constraint can be done efficiently, as was explained in the previous section. While $Q(w)$ is a local model, it can capture nonconvexity of the true function $f(w)$ and hence allows the algorithm to follow negative curvature directions. As we will see in \S\ref{BO}, for the same amount of number of function evaluations, DFO-TR achieved better or comparable function values, while requiring significantly less computational time than Bayesian optimization algorithms (TPE, SMAC, and SPEARMINT).

\begin{algorithm}[ht]
\caption{~\textbf{Bayesian Optimization}}
\label{Bayes}
\begin{algorithmic}
\STATE {1:} \textbf{for $t = 1,2, \cdots$ do} 
\STATE {2:} ~~~~Find $w_k$ by optimizing the acquisition function over \\
~~~~~~~model $M$: $w_k \gets \arg \min _{w} a_M(w|\mathcal{D}_{1:k-1}).$
 \STATE  {3:} ~~~Sample the objective function: $v_k := f(w_k)$.
 \STATE  {4:} ~~~Augment the data $\mathcal{D}_{1:k} = \{\mathcal{D}_{1:k-1}, (w_k,v_k)  \}$\\
 ~~~~~~~and update the model $M$.
 \STATE  {5:} \textbf{end for}
 \end {algorithmic}
\end{algorithm}

\section{Numerical Experiments}\label{third}
\subsection{Optimizing Smooth, NonConvex Benchmark Functions} \label{benchmarks}
In this section, we compare the performance of DFO-TR and Bayesian optimization algorithms on optimizing three nonconvex smooth benchmark functions. We compare the precision $\Delta f_{opt}$ with the global optimal  value, which is known,  and is computed after a given number of function evaluations. 

%This comparison is in terms of the number of function evaluations to reach the global optimal function value. In what follows, $\Delta f_{opt}$ is the precision to reach, that is the difference to the global optimal function value, which is computed after a given number of function evaluations. 

Algorithm \ref{al1} is implemented in Python 2.7.11 \footnote{\url{https://github.com/TheClimateCorporation/dfo-algorithm}} . We start from the zero vector as the initial point. In addition, the trust region radius is initialized as $\Delta_0 = 1$ and the initial interpolation set has $d +1$ random members. The parameters are chosen as $\eta_0 = 0.001$, $\eta_1 = 0.75$, $\theta=10$, $\gamma_1 = 0.98$, and $\gamma_2 = 1.5$. 
We have used the hyperparameter optimization library, HPOlib \footnote{\url{www.automl.org/hpolib}}, to perform the experiments on TPE, SMAC, and SPEARMINT algorithms, implemented in Python, Java3, and MATLAB, respectively. Each benchmark function is evaluated on its known search space, as is defined in the default setting of the HPOlib (note that DFO-TR does not require nor utilizes a restricted search space). 

We can see that on all three problems DFO-TR reaches the global value accurately and quickly, outperforming BO methods. This is because DFO-TR utilizes second-order information effectively, which helps following negative curvature and significantly improving convergence in the absence of noise.
Among the three Bayesian optimization algorithms SPEARMINT performs better while the performance of TPE and SMAC is comparable to each other, but inferior to those of SPEARMINT and DFO-TR.

%\textcolor{red} {HIVA!!: EXPLAIN DFO DOES NOT NEED INITIAL SEARCH SPACE WHILE HPOLIB USES THE BEST KNOWN SEARCH SPACE IN THE FIRST PLACE}

% Branin
\begin{table}[ht]     
 \caption{DFO-TR vs. BO on \textit{Branin} function in terms of $\Delta f_{opt}$, over number of function evaluations. \textit{Branin} is a two dimensional function with $f_{opt} = 0.397887$.}
  \label{t2}
 % \vskip 0.15in
\begin{center}
\begin{small}
%\begin{sc}
  \begin{tabular}{ ccccr}  
  \hline
  %\abovespace\belowspace
{Algorithm}&{1}&{5}&\textbf{11}&{100}\\ 
 \hline
% \abovespace
DFO-TR&15.7057&0.1787&0&0\\  
TPE&30.0880&4.8059&3.4743&0.0180\\ 
SMAC&23.7320&10.3842&6.7017&0.0208\\ 
%\belowspace
SPEARMINT&34.3388&17.1104&1.1615&3.88e-08\\ 
\hline
\end{tabular}
%\end{sc}
\end{small}
\end{center}
\vskip -0.2in
\end{table}

% Camelback
\begin{table}[ht]     
\tiny
 \captionsetup{justification=centering}
 \caption{DFO-TR vs. BO on \textit{Camelback} function in terms of $\Delta f_{opt}$, over number of function evaluations. \textit{Camelback} is a two dimensional function with $f_{opt} = -1.031628$.}
  \label{t2}
  %\vskip 0.15in
\begin{center}
\begin{small}
%\begin{sc}
  \begin{tabular}{ ccccr}  
  \hline
 % \abovespace\belowspace
{Algorithm}&{1}&{10}&\textbf{21}&{100}\\  
\hline
 %\abovespace
DFO-TR&2.3631&0.1515&0&0\\  
TPE&3.3045&0.5226&0.3226&0.0431\\ 
SMAC&1.0316&0.0179&0.0179&0.0036\\ 
%\belowspace
SPEARMINT&2.3868&1.6356&0.1776&2.29e-05\\ 
\hline
\end{tabular}
%\end{sc}
\end{small}
\end{center}
\vskip -0.2in
\end{table}

% Hartmann
\begin{table}[ht]     
\tiny
 \caption{DFO-TR vs. BO on \textit{Hartmann} function in terms of $\Delta f_{opt}$, over number of function evaluations. \textit{Hartmann} is a six dimensional function with $f_{opt} = -3.322368$.}
  \label{t2}
% \vskip 0.15in
\begin{center}
\begin{small}
%\begin{sc}
  \begin{tabular}{ ccccr}  
  \hline
 % \abovespace\belowspace
{Algorithm}&{1}&{25}&\textbf{64}&{250}\\  
\hline
% \abovespace
DFO-TR&3.1175&0.4581&0&0\\  
TPE&3.1862&2.5544&1.4078&0.4656\\ 
SMAC&2.8170&1.5311&0.6150&0.2357\\ 
% \belowspace
SPEARMINT&2.6832&2.6671&2.5177&9.79e-05\\ 
\hline
\end{tabular}
%\end{sc}
\end{small}
\end{center}
\vskip -0.2in
\end{table}

\subsection{Optimizing AUC Function} \label{BO}
In this section, we compare the  performance of {DFO-TR} and  the three Bayesian optimization algorithms, TPE, SMAC, and SPEARMINT, on the task of 
optimizing AUC of a linear classifier, defined by $w$. While in \S\ref{first}, we have argued that $\mathbb{E}[F_{AUC}(w)]$ is a smooth function, 
in practice we have a finite data set, hence we compute the noisy nonsmooth estimate of $\mathbb{E}[F_{AUC}(w)]$. This, essentially means, that
we can only expect to optimize the objective up to some accuracy, after which the noise will prevent further progress. 
%Therefore, as the first approach to handle  this difficulty, we use the noisy nonsmooth function $F_{AUC}(w)$--stated in \eqref{AUC}--for the two given finite sets ${\cal S_+}$ and ${\cal S_-}$, as an approximation of possibly smooth function $\mathbb{E}[F_{AUC}(w)]$. In particular, $F_{AUC}(w)$ plays the role of the black-box function, so that Algorithm \ref{al1} is performed to find a local minimum of $\mathbb{E}[F_{AUC}(w)]$ over the space of all $w$ vectors. \textcolor{red}{KATYA!!!!!!}

In our experiments, we used 12 binary class data sets, as shown in Table \ref{t1}, where  some of the data sets are normalized  so that each dimension has mean 0 and variance 1. These data sets can be downloaded from LIBSVM website \footnote{\url{https://www.csie.ntu.edu.tw/~cjlin/libsvmtools/datasets/binary.html}} and UCI machine learning repository. %\footnote{\url{http://archive.ics.uci.edu/ml/}}.
 %  \footnote{\url{https://www.csie.ntu.edu.tw/~cjlin/libsvmtools/datasets/binary.html}}
 
 % Table 5.4
\begin{table}[ht]  
\caption{data sets statistic, $d:$ dimension, $N:$ number of data points, $N_{-}/N_{+}:$ class distribution ratio,  $AC:$ attribute characteristics.}
\label{t1}
%\vskip 0.15in
\begin{center}
\begin{small}
%\begin{sc}
 \begin{tabular}{cccccr} 
\hline
%\abovespace\belowspace
\textbf{data set}&\textbf{AC}&\textbf{$d$}&\textbf{$N$}& $N_{-}/N_{+}$\\ 
\hline
%\abovespace
fourclass&$[-1,1]$&2&862&1.8078\\
magic04&$[-1,1]$, scaled&10&19020&1.8439\\
svmguide1&$[-1,1]$, scaled&4&3089&1.8365\\
diabetes&$[-1,1]$&8&768&1.8657\\
german&$[-1,1]$&24&1000&2.3333\\
svmguide3&$[-1,1]$, scaled&22&1243&3.1993\\
shuttle&$[-1,1]$, scaled&9&43500&3.6316\\
segment&$[-1,1]$&19&2310&6\\
ijcnn1&$[-1,1]$&22&35000&9.2643\\
satimage&$[27,157]$, integer&36&4435&9.6867\\
vowel&$[-6,6]$&10&528&10\\
%\belowspace
letter&$[0,15]$, integer&16&20000&26.2480\\ 
\hline
\end{tabular}
%\end{sc}
\end{small}
\end{center}
\vskip -0.2in
\end{table}

The average value of AUC and its standard deviation, using five-fold cross-validation, is reported as the performance measure. Table~\ref{t2} summarizes the results. 

The initial vector $w_0$ for DFO-TR  is set to zero and the search space of Bayesian optimization algorithms is set to interval $[-1,1]$. For each data set, a fixed total budget of number of function evaluations is given to each algorithm and the final AUC computed on the test set is compared.

\begin{table*}[t!]     
 \caption{Comparing DFO-TR vs. BO algorithms.}
  \label{t2}
   \vskip 0.05in
\begin{center}
\begin{small}
%\begin{sc}
 \begin{tabular}{ c|c|cc|cc|cc|cc}  
 \hline
  %\abovespace\belowspace
\multirow{2}{*}{\textbf{Data}} 
& \multicolumn{1}{c|}{\textbf{num.}}&\multicolumn{2}{c|}{\textbf{DFO-TR}}&\multicolumn{2}{c|}{\textbf{TPE}}&\multicolumn{2}{c|}{\textbf{SMAC}}&\multicolumn{2}{c}{\textbf{SPEARMINT}}\\   
&\textbf{fevals}&{AUC}&{time}&{AUC}&{time}&{AUC}&{time}&{AUC}&{time}\\  
\hline
%\abovespace
fourclass&100&0.835$\pm$0.019&0.31&\textbf{0.839}$\pm$0.021&12&0.839$\pm$0.021&77&0.838$\pm$0.020&5229\\ 
svmguide1&100&0.988$\pm$0.004&0.71&0.984$\pm$0.009&13&0.986$\pm$0.006&72&\textbf{0.987}$\pm$0.006&6435\\ 
diabetes&100&0.829$\pm$0.041&0.58&0.824$\pm$0.044&15&0.825$\pm$0.045&75&\textbf{0.829}$\pm$0.060&8142\\ 
shuttle&100&0.990$\pm$0.001&43.4&0.990$\pm$0.001&17&0.989$\pm$0.001&76&\textbf{0.990}$\pm$0.001&13654\\ 
vowel&100&0.975$\pm$0.027&0.68&0.965$\pm$0.029&16&0.965$\pm$0.038&77&\textbf{0.968}$\pm$0.025&9101\\ 
magic04&100&0.842$\pm$0.006&10.9&0.824$\pm$0.009&16&0.821$\pm$0.012&76&\textbf{0.839}$\pm$0.006&7947\\ 
letter&200&0.987$\pm$0.003&10.2&0.959$\pm$0.008&49&0.953$\pm$0.022&166&\textbf{0.985}$\pm$0.004&21413\\ 
segment&300&0.992$\pm$0.007&9.1&0.962$\pm$0.021&99&\textbf{0.997}$\pm$0.004&263&0.976$\pm$0.021&216217\\ 
ijcnn1&300&0.913$\pm$0.005&57.3&0.677$\pm$0.015&109&0.805$\pm$0.031&268&\textbf{0.922}$\pm$0.004&259213\\ 
svmguide3&300&0.776$\pm$0.046&13.5&0.747$\pm$0.026&114&\textbf{0.798}$\pm$0.035&307&0.7440$\pm$0.072&185337\\ 
german&300&0.795$\pm$0.024&9.9&0.771$\pm$0.022&120&0.778$\pm$0.025&310&\textbf{0.805}$\pm$0.020&242921\\ 
%\belowspace
satimage&300&0.757$\pm$0.013&14.2&0.756$\pm$0.020&164&0.750$\pm$0.011&341&\textbf{0.761}$\pm$0.028&345398\\   
\hline
\end{tabular}
%\end{sc}
\end{small}
\end{center}
\vskip -0.2in
\end{table*}

For each data set, the bold number indicates the best average AUC value found by a Bayesian optimization algorithms. We can see that DFO-TR attains comparable or better AUC value to the best one, in almost all cases. Since for each data set, all algorithms are performed for the same budget of number of function evaluations, we do not include the time spent on function evaluations in the reported time. Thus, the time reported in Table~\ref{t2} is only the optimizer time. As we can see, DFO-TR is significantly faster than Bayesian optimization algorithms, while it performs competitively in terms of the average value of AUC. Note that the problems are listed in the order of increasing dimension $d$. Even thought the MATLAB implantation of SPEARMINT probably puts it at a certain disadvantage in terms of computational time comparisons, we observe that it is clearly a slow method, whose complexity grows significantly as  $d$ increases. 

%In Figures~\ref{fig1} and \ref{fig2} illustrate the progress of each of the algorithms in terms of best AUC value vs. the number of  function evaluations. As we can see, for all the data sets, the growth rate of DFO-TR is faster or as competitive as the best Bayesian optimization algorithm. In particular, for all the data sets, DFO-TR surpasses the Bayesian optimization algorithms at the early stage and remains competitive throughout. 

Next, we compare the performance of DFO-TR versus the random search algorithm (implemented in Python 2.7.11) on maximizing AUC. Table~\ref{t3} summarizes the results, in a similar manner to Table~\ref{t2}. Moreover, in Table~\ref{t3}, we also allow  random search to use twice the budget of the function evaluations, as is done in \cite{kevin} when comparing random search to BO. The random search algorithm is competitive with DFO-TR on a few problems, when using twice the budget, however, it can be seen that as the problem dimension grows, the efficiency of random search goes down substantially.  Overall, DFO-TR consistently surpasses random search  when function evaluation budgets are equal, while not requiring very significant overhead, as the BO methods. 

% Table 5.6
\begin{table*}[t!]     
 \centering
 %\tiny
 \captionsetup{justification=centering}
 \caption{Comparing DFO-TR vs. random search algorithm.}
  \label{t3}
 \vskip 0.05in
\begin{center}
\begin{small}
%\begin{sc}
  \begin{tabular}{ c|cc|cc|cc}  
  \hline
  %\abovespace\belowspace
\multirow{2}{*}{\textbf{Data}} 
& \multicolumn{2}{c|}{\textbf{DFO-TR}}&\multicolumn{2}{c|}{\textbf{Random Search}}&\multicolumn{2}{c}{\textbf{Random Search}}\\   
&AUC&{num. fevals}&{AUC}&{num. fevals}&{AUC}&{num. fevals}\\ 
 \hline
% \abovespace
fourclass&0.835$\pm$0.019&100&0.836$\pm$0.017&100&0.839$\pm$0.021&200\\  
svmguide1&0.988$\pm$0.004&100&0.965$\pm$0.024&100&0.977$\pm$0.009&200\\   
diabetes&0.829$\pm$0.041&100&0.783$\pm$0.038&100&0.801$\pm$0.045&200\\   
shuttle&0.990$\pm$0.001&100&0.982$\pm$0.006&100&0.988$\pm$0.001&200\\   
vowel&0.975$\pm$0.027&100&0.944$\pm$0.040&100&0.961$\pm$0.031&200\\   
magic04&0.842$\pm$0.006&100&0.815$\pm$0.009&100&0.817$\pm$0.011&200\\   
letter&0.987$\pm$0.003&200&0.920$\pm$0.026&200&0.925$\pm$0.018&400\\   
segment&0.992$\pm$0.007&300&0.903$\pm$0.041&300&0.908$\pm$0.036&600\\   
ijcnn1&0.913$\pm$0.005&300&0.618$\pm$0.010&300&0.629$\pm$0.013&600\\ 
svmguide3&0.776$\pm$0.046&300&0.690$\pm$0.038&300&0.693$\pm$0.039&600\\   
german&0.795$\pm$0.024&300&0.726$\pm$0.028&300&0.739$\pm$0.021&600\\   
% \belowspace
satimage&0.757$\pm$0.013&300&0.743$\pm$0.029&300&0.750$\pm$0.020&600\\   
\hline
\end{tabular}
%\end{sc}
\end{small}
\end{center}
\vskip -0.2in
\end{table*}

We finally note that while we exclude  comparisons with other methods, that optimize AUC surrogates, from this paper, due to lack of space, we have performed such experiments and observed that the final AUC values obtained by DFO-TR and BO are competitive with other existing methods. 

\subsubsection{Stochastic versus Deterministic DFO-TR} \label{subsub}
In order to further improve efficiency of DFO-TR, we observe that STORM framework and theory \cite{matt} suggests that noisy function evaluations do not need to be accurate far away from the optimal solution. In our context, this means that AUC can be evaluated on small subsets of the training set, which gradually increase as the algorithm progresses. 
%As the second approach to approximate true $\mathbb{E}[F_{AUC}(w)]$, throughout Algorithm~\ref{al1}, we consider the sets $\mathcal{S}_+$ and $\mathcal{S}_-$ as the whole space of positive and negative samples, respectively. Then, we randomly choose i.i.d. sample points from these two sets to compute $F_{AUC}(w)$, as our black-box function. In this case, the value of $F_{AUC}(w)$ can be interpreted as the expected value of AUC, with respect to sampling from the given data sets $\mathcal{S}_+$ and $\mathcal{S}_-$. Thus, in contrast to our first approach, $F_{AUC}(w)$ is not the exact deterministic value of AUC, for the given data sets $\mathcal{S}_+$ and $\mathcal{S}_-$. Toward our second approach, we present a stochastic variant of Algorithm~\ref{al1} with the following modifications:
In particular, at each iteration, we compute AUC on a subset of data, which is sampled from positive and negative sets uniformly at random, at the rate
\begin{equation*} \label{update1}
\begin{aligned}
 \min & \{ N, \max  \{ k \times  \lfloor 50 \times ({N}/{(N_++N_-)}) \rfloor + \\
 & \lfloor 1000 \times ({N}/{(N_++N_-)}) \rfloor ,  \lfloor{0.1 \times N} \rfloor  \} \}, 
\end{aligned}
\end{equation*}
where $N_+=\mathcal{S}_+$ and $N_-=\mathcal{S}_-$, and $N=N_+$, when we sample from the positive class and $N=N_-$, when we sample from the negative one. 
For each class, at least 10 percent of the whole training data is used. 

We include an additional modification--after each unsuccessful step with $\rho_k < \eta_0$, we compute $f_{new}(w_k)$ by resampling over data points. Then, we update $f(w_k)$ such that $f(w_k) := \left(f(w_k) + f_{new}(w_k)\right)/2$. This is done, so that accidental incorrectly high AUC values are not preventing the algorithm from making progress. 
This results in a less expensive (in terms of function evaluation cost) algorithm, while, as we see in 
 Figure~\ref{fig3}, the convergence to the optimal solution is comparable.  
 
 We chose two data sets {\em shuttle} and {\em letter}  to compare the performance of the stochastic variant of the DFO-TR with the deterministic one. These sets were chosen because they contain a relatively large number of data points and hence the effect of subsampling can be observed. 
We repeated each experiment  four times using five-fold cross-validation (due to the random nature of the stochastic sampling). Hence, for each problem,
the algorithms have been applied 20 times in total, and the average AUC values are reported in Figure~\ref{fig3}. At each round, all parameters of DFO-TR and S-DFO-TR are set as described in \S\ref{benchmarks}, except $w_0$, which is a random vector evenly distributed over $[-1,1]$.

As we see in Figure~\ref{fig3}, the growth rate of AUC over iterations in S-DFO-TR is as competitive as that of DFO-TR. However, by reducing the size of the data sets,  the iteration of S-DFO-TR are significantly cheaper than that of DFO-TR, especially at the beginning. This indicates that the methods can handle large data sets. 

We finally note that we chose to optimize AUC over linear classifiers for simplicity only. Any other classifier parametrized by $w$ can be trained using a black-box optimizer in a similar way. However, the current DFO-TR method have some difficulties in convergence with problems when dimension of $w$ 
is very large. 

\begin{figure} [ht]
\begin{subfigure} 
  \centering
  \includegraphics[height=0.35\linewidth,width=0.48\linewidth]{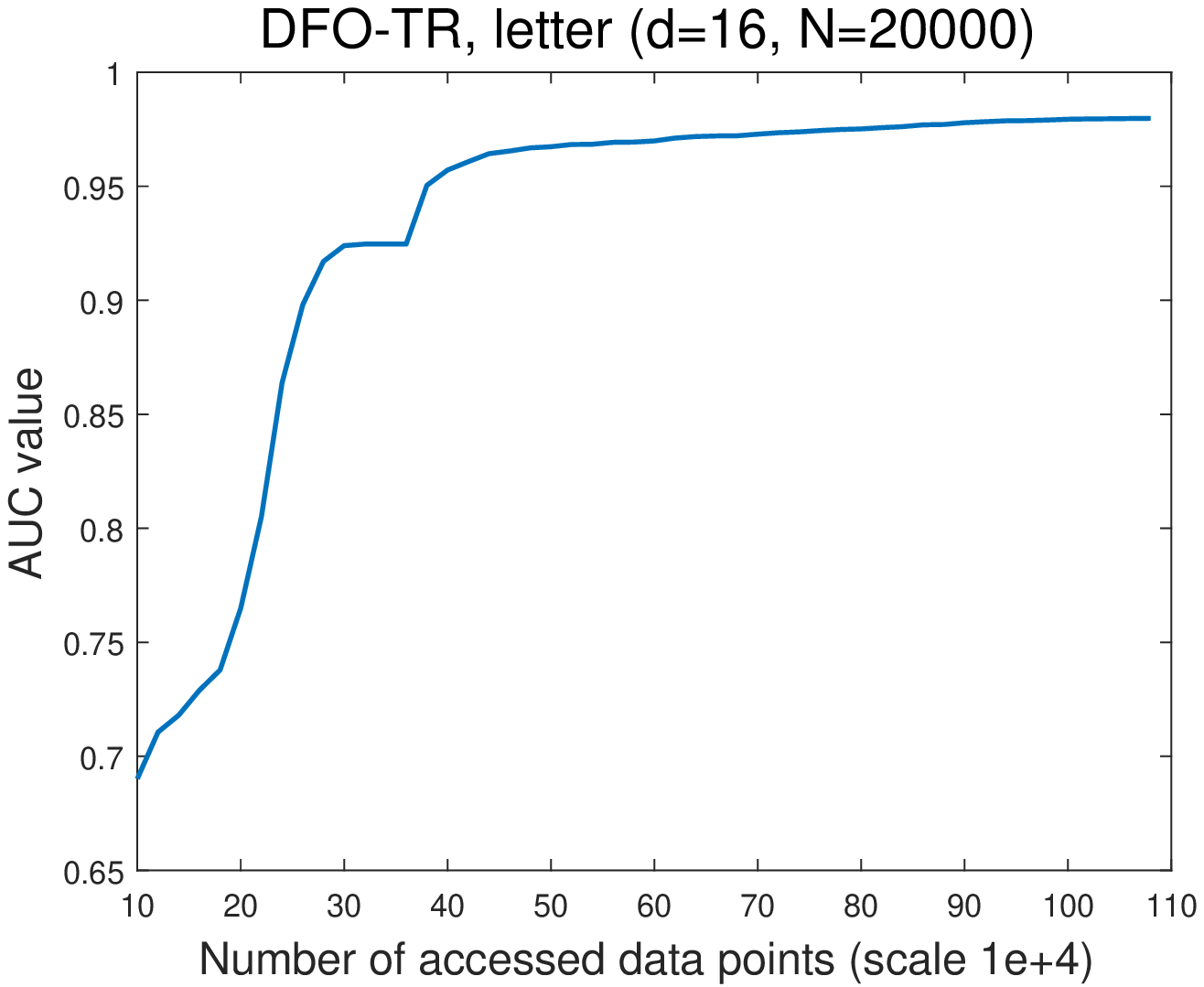}
 % \caption{DFO-TR, letter (d=16, N=20000)}
  \label{fig:sfig1}
\end{subfigure}%
\begin{subfigure}
  \centering
  \includegraphics[height=0.35\linewidth,width=0.49\linewidth]{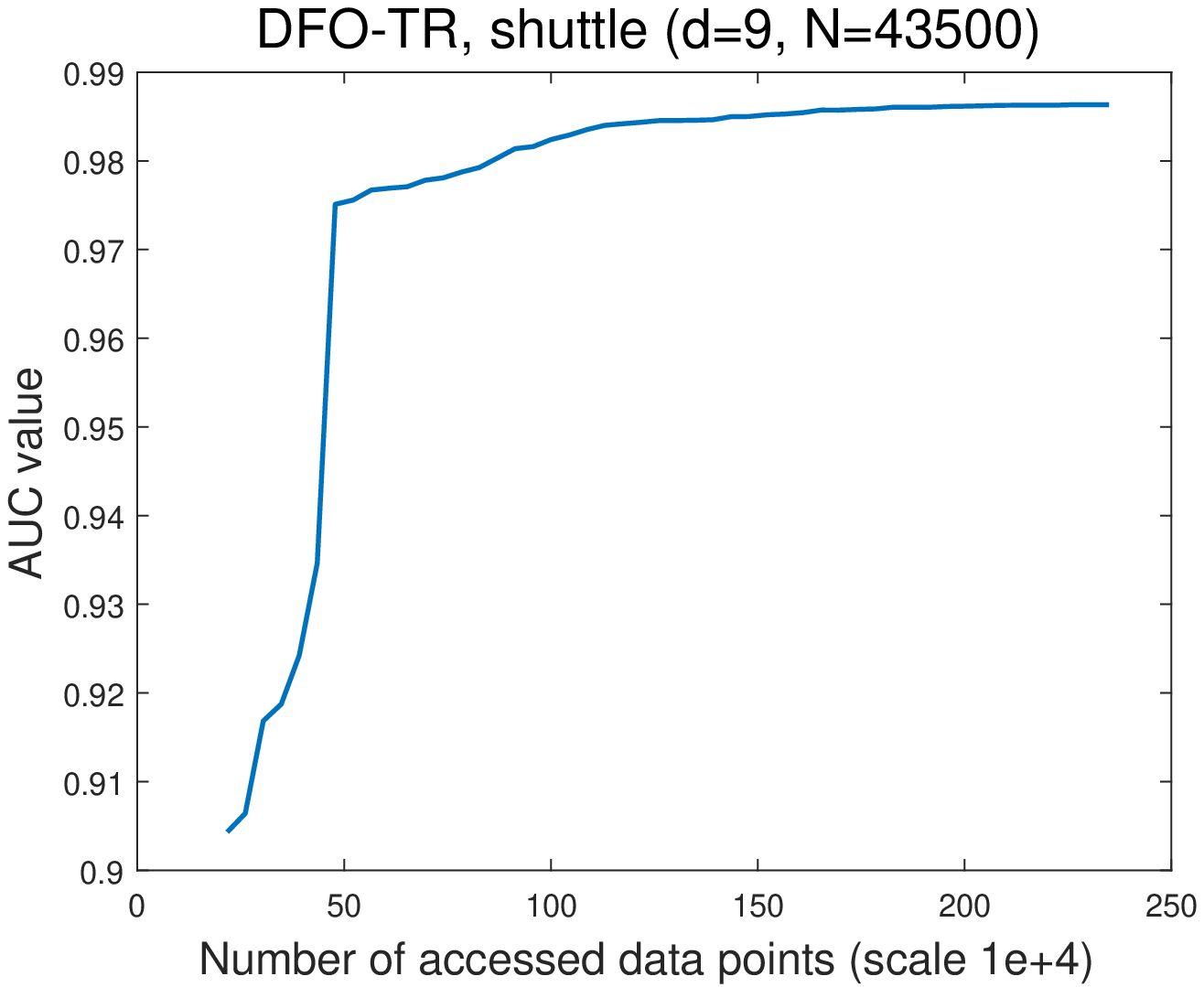}
  %\caption{S-DFO-TR, letter (d=16, N=20000)}
    \label{fig:sfig2}
\end{subfigure}
\begin{subfigure}%
  \centering
  \includegraphics[height=0.35\linewidth,width=0.48 \linewidth]{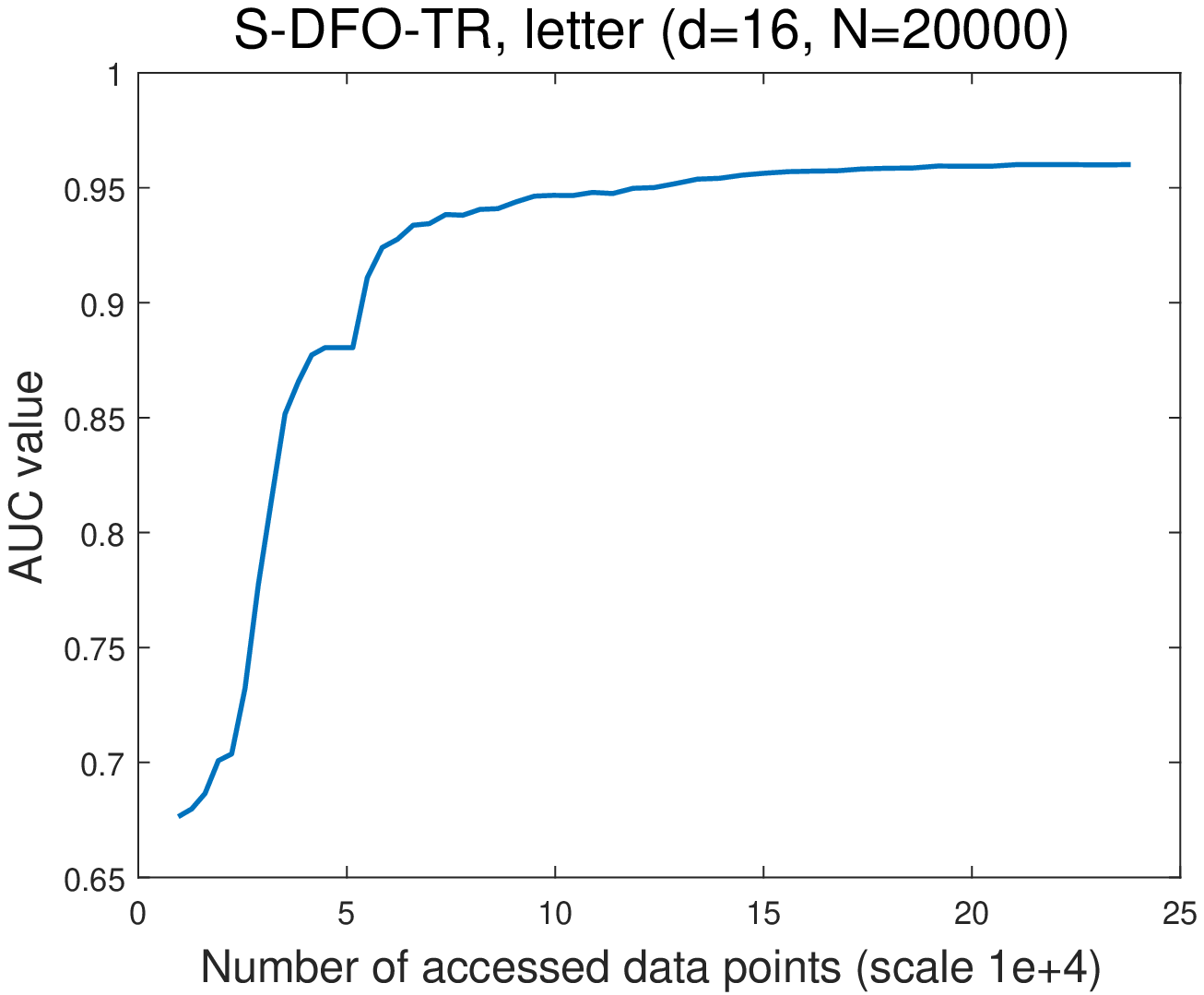}
%\caption{DFO-TR, shuttle (d=9, N=43500)}
  \label{fig:sfig2}
\end{subfigure} %
\begin{subfigure}
  \centering
  \includegraphics[height=0.35\linewidth,width=0.48\linewidth]{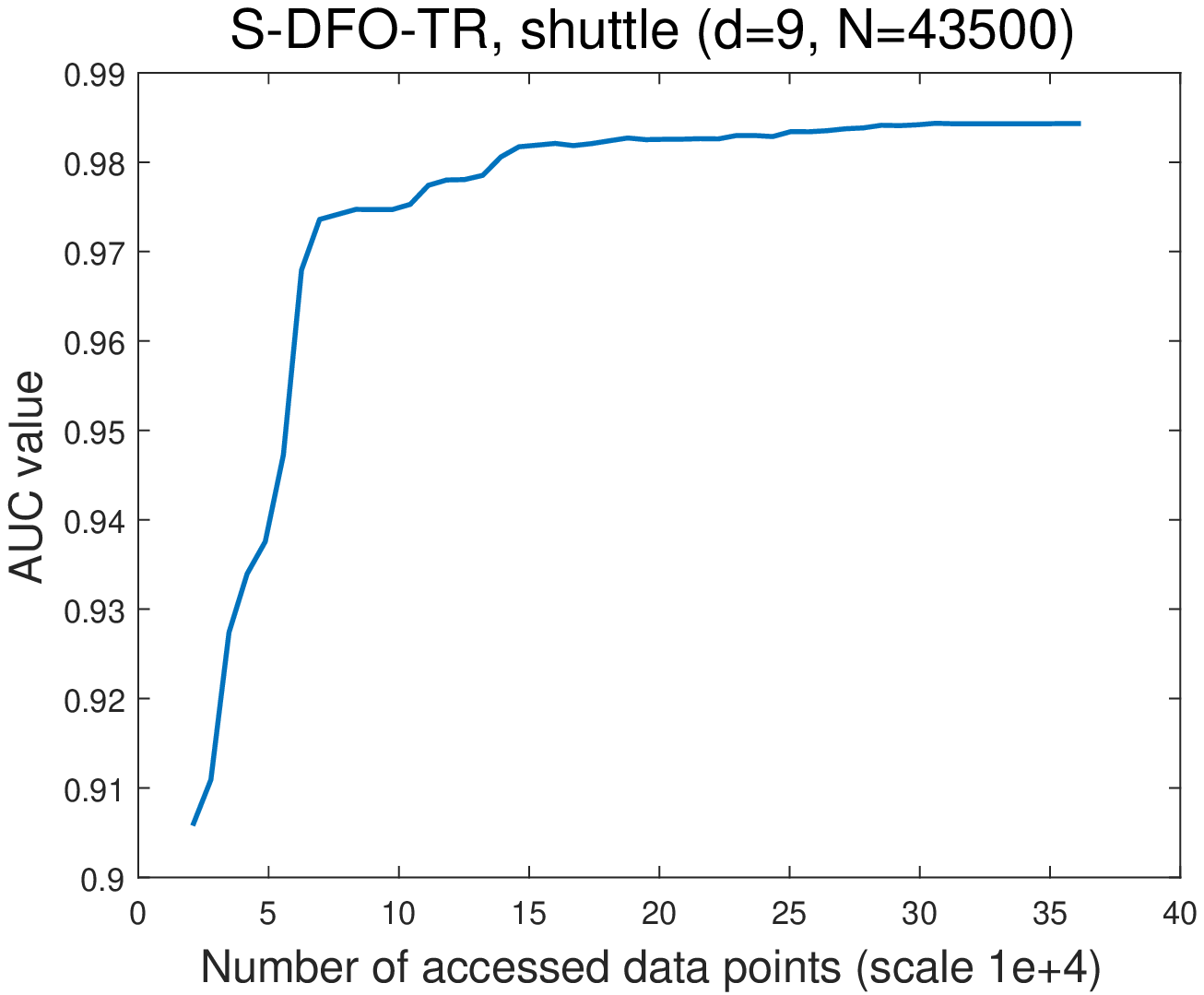}
  %\caption{DFO-TR, shuttle (d=9, N=43500)}
  \label{fig:sfig2}
\end{subfigure}
\caption{Comparison of stochastic {DFO-TR} and deterministic one in optimizing AUC function.}
\label{fig3}
\end{figure}

\subsection{Hyperparameter Tuning  of Cost-Sensitive RBF-Kernel SVM} \label{RBF_SVM}

Finally, we turn to hyperparamater tuning to show that DFO-TR can also outperform state-of-the-art methods on this problem. We consider tuning parameters of an RBF-kernel, cost-sensitive, SVM, with $\ell_2$ regularization parameter $\lambda$, kernel width $\gamma$, and positive class cost $c_p$. Thus, in this setting, we compare the performance of DFO-TR, random search, and Bayesian optimization algorithms, in tuning a three-dimensional hyperparameter $w = (\lambda,\gamma,{c_p})$, in order to achieve a high test accuracy. 

For the random search algorithm, as well as the Bayesian optimization algorithms, the search space is chosen as $\lambda \in [10^{-6},10^0]$, $\gamma \in [10^0,10^3]$,  as is done in \cite{kevin}, and $c_p \in [10^{-2},10^{2}]$. The setting of Algorithm \ref{al1} is as described in \S\ref{benchmarks}, while $w_0 = (\lambda_0,\gamma_0,{c_p}_0)$ is a three-dimensional vector randomly drawn from the search space defined above.

We have used the five-fold cross-validation with the \textit{train-validate-test} framework as follows: we used two folds as the training set for the SVM model, other two  folds as the validation set to compute and maximize the \textit{validation}  accuracy, and the remaining one as the test set to report the \textit{test} accuracy.  

%To be more specific, consider an algorithm $\mathcal{A}$, containing hyperparameters $\lambda, \gamma,$ and  $c_p$ with respective domains $\mathcal{X}_\lambda = [10^{-6},10^0], \mathcal{X}_\gamma = [10^0,10^3],$ and $\mathcal{X}_{c_p} = [10^{-2},10^{2}]$. If we define a three-dimensional hyperparameter $x = (\lambda,\gamma,c_p)$, then we will have hyperparameter space $\mathcal{X} = \mathcal{X}_\lambda \times \mathcal{X}_\gamma \times \mathcal{X}_{c_p}$. Now, $\mathcal{A}_{\bar{x}}$ denotes the algorithm $\mathcal{A}$ under fixed hyperparameter setting $\bar{x} \in \mathcal{X}$. Now, assume that we split the data set $\mathcal{D}$ into three different sets $\mathcal{D}_{train}$, $\mathcal{D}_{valid}$, and $\mathcal{D}_{test}$, respectively as \textit{training}, \textit{validation} and \textit{test} data set. In this setting, we define the validation loss $f_{\mathcal{L}}(\mathcal{A}_{\bar{x}}, \mathcal{D}_{train},\mathcal{D}_{valid})$, as the loss which algorithm  $\mathcal{A}_{\bar{x}}$ obtains on $\mathcal{D}_{valid}$ after being trained on $\mathcal{D}_{train}$. Under five-fold cross validation, the hyperparameter optimization problem can be interpreted as the black-box optimization, aiming to minimize the following black-box function $F(x) = \frac{1}{5} \sum_{i=1}^5 f_{\mathcal{L}} (\mathcal{A}_{\bar{x}}, \mathcal{D}^{(i)}_{train},\mathcal{D}^{(i)}_{valid})$.
%\begin{equation} \label{blackbox}
%F(x) = \frac{1}{5} \sum_{i=1}^5 f_{\mathcal{L}} (\mathcal{A}_{\bar{x}}, \mathcal{D}^{(i)}_{train},\mathcal{D}^{(i)}_{valid}).
%\end{equation}

Figure \ref{fig svm} illustrates the performance of DFO-TR versus random search and Bayesian optimization algorithms, in terms of the average test accuracy over the number of function evaluations. As we can see, DFO-TR constantly surpasses random search and Bayesian optimization algorithms. 
It is worth mentioning that random search is competitive with the BO methods and in contrast to \S\ref{benchmarks} and \S\ref{BO}, SMAC performs the best among the Bayesian optimization algorithms.

% \begin{figure}[H]
%\centering
%%\begin{tabular}{c:c:c:c}
% \epsfig{file=../images/SVM_diabetes,width=0.23\textwidth}
%   \epsfig{file=../images/SVM_fourclass,width=0.23\textwidth}
% \epsfig{file=../images/SVM_magic,width=0.23\textwidth}
%  \epsfig{file=../images/SVM_satimage,width=0.23\textwidth}
% \caption{\footnotesize \new{Comparison of DFO-TR, random search and Bayesian Optimization on tuning RBF-kernel SVM hyperparameters}}
%   \label{fig:ms}
% \end{figure}
%

%%%%% Figure 5.4 %%%%%%
\begin{figure} [ht]
\begin{subfigure}
  \centering
  \includegraphics[height=0.35\linewidth,width=0.49\linewidth]{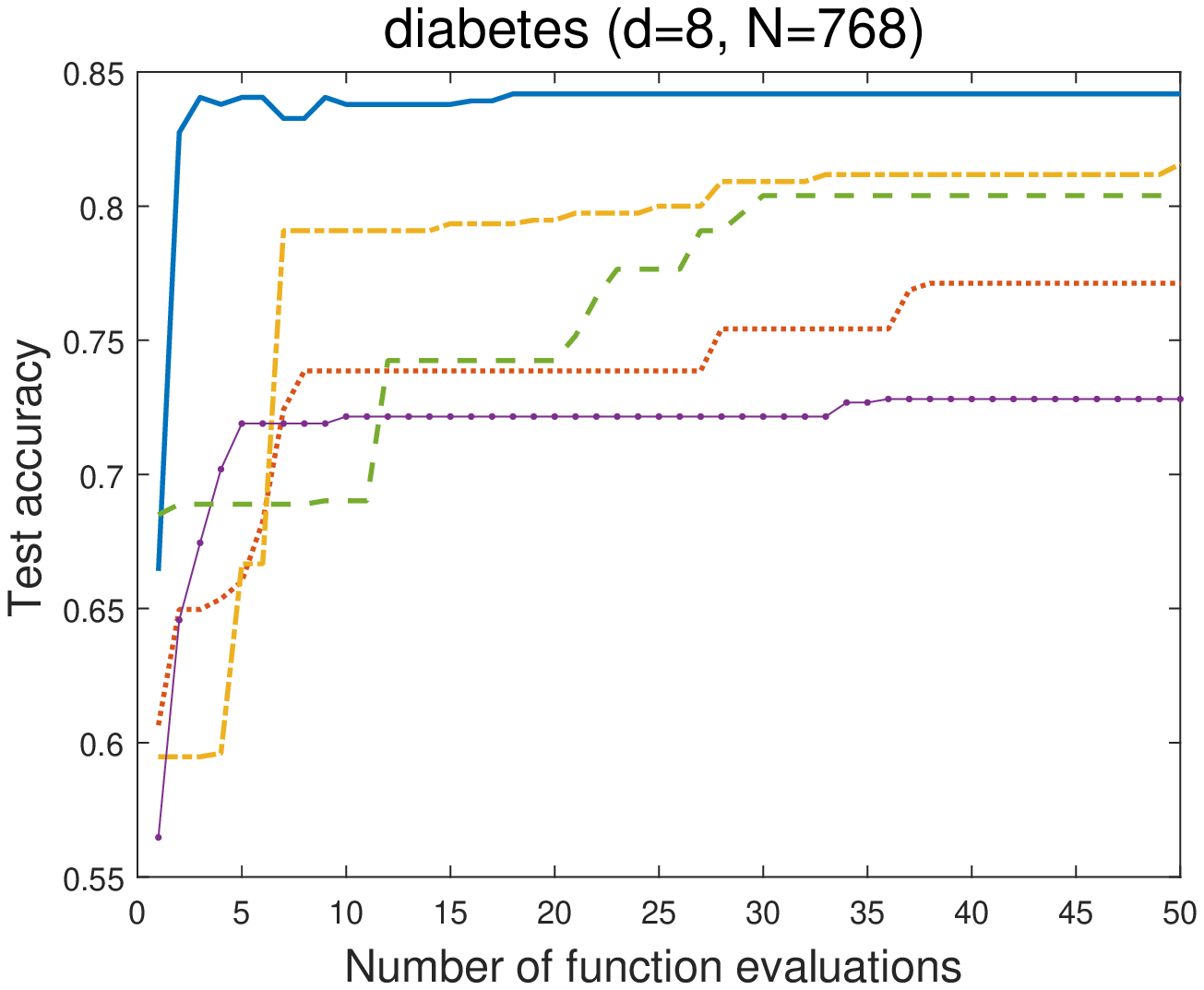}
%  \caption{diabetes (d=8, N=768)}
  \label{fig:sfig1}
\end{subfigure}%
\begin{subfigure}
  \centering
  \includegraphics[height=0.35\linewidth,width=0.5\linewidth]{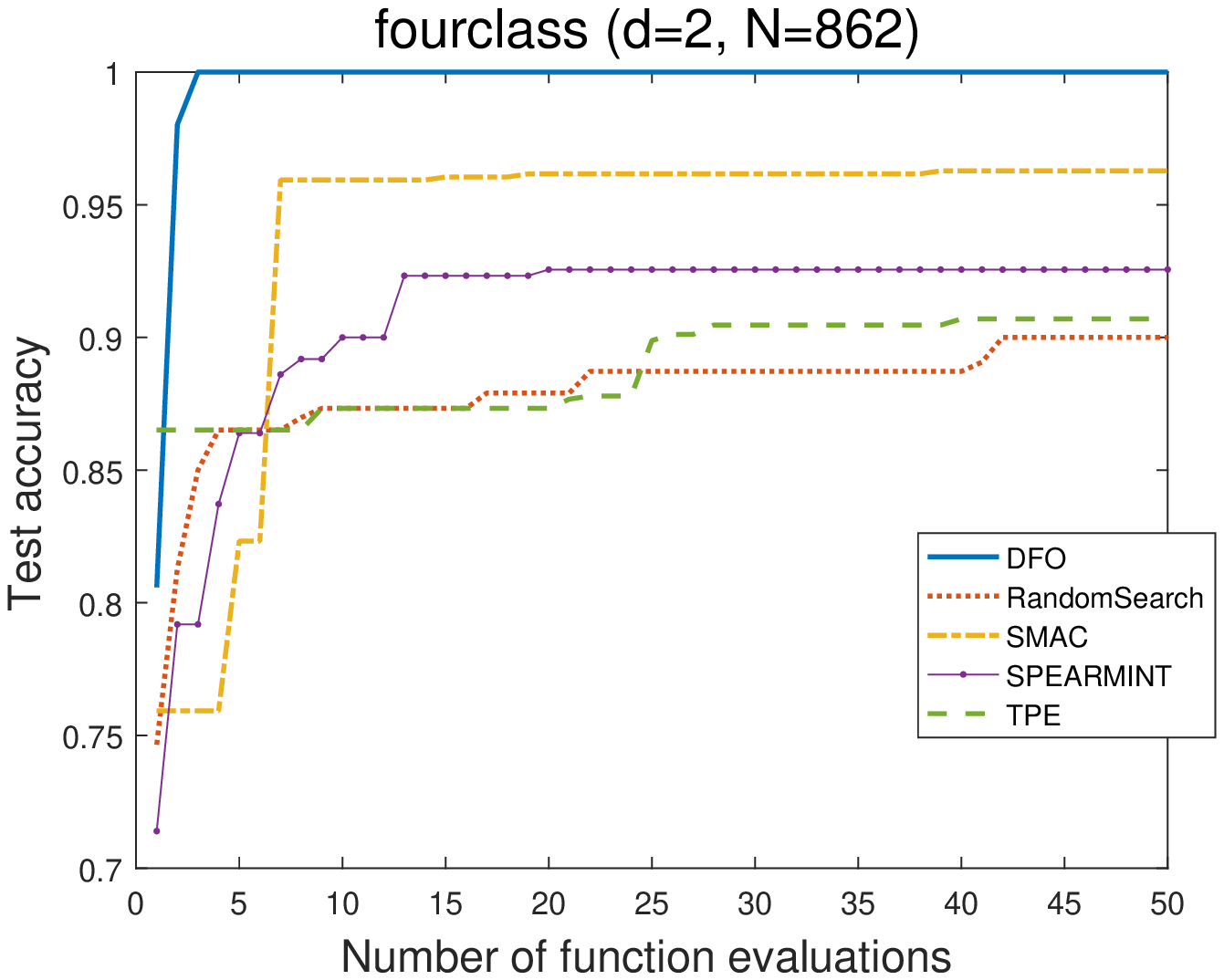}
%  \caption{diabetes (d=8, N=768)}
  \label{fig:sfig1}
\end{subfigure}%
\begin{subfigure}
 \centering
  \includegraphics[height=0.35\linewidth,width=0.49\linewidth]{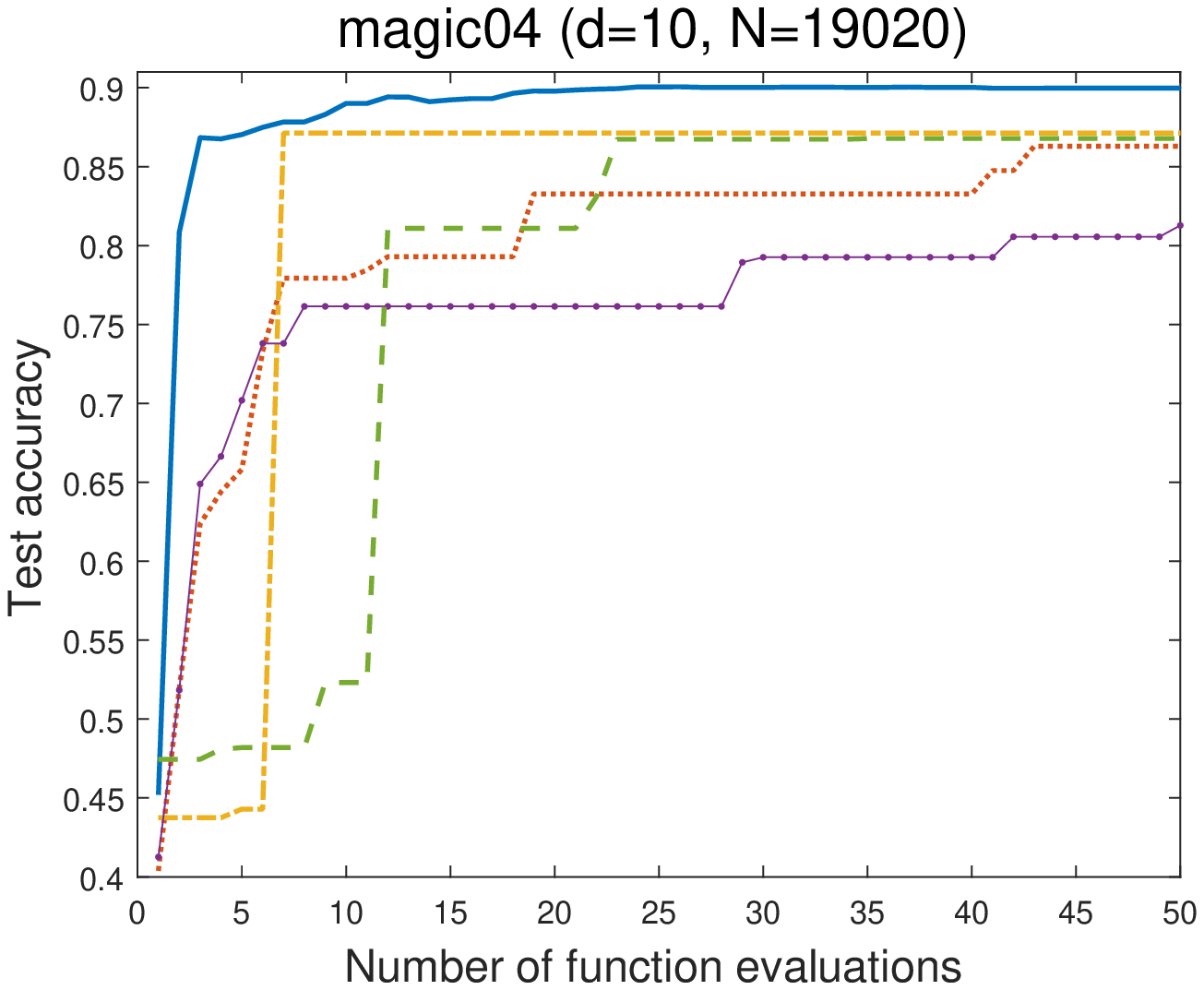}
 % \caption{magic (d=10, N=19020)}
  \label{fig:sfig5}
\end{subfigure}%
\begin{subfigure}
  \centering
  \includegraphics[height=0.35\linewidth,width=0.49\linewidth]{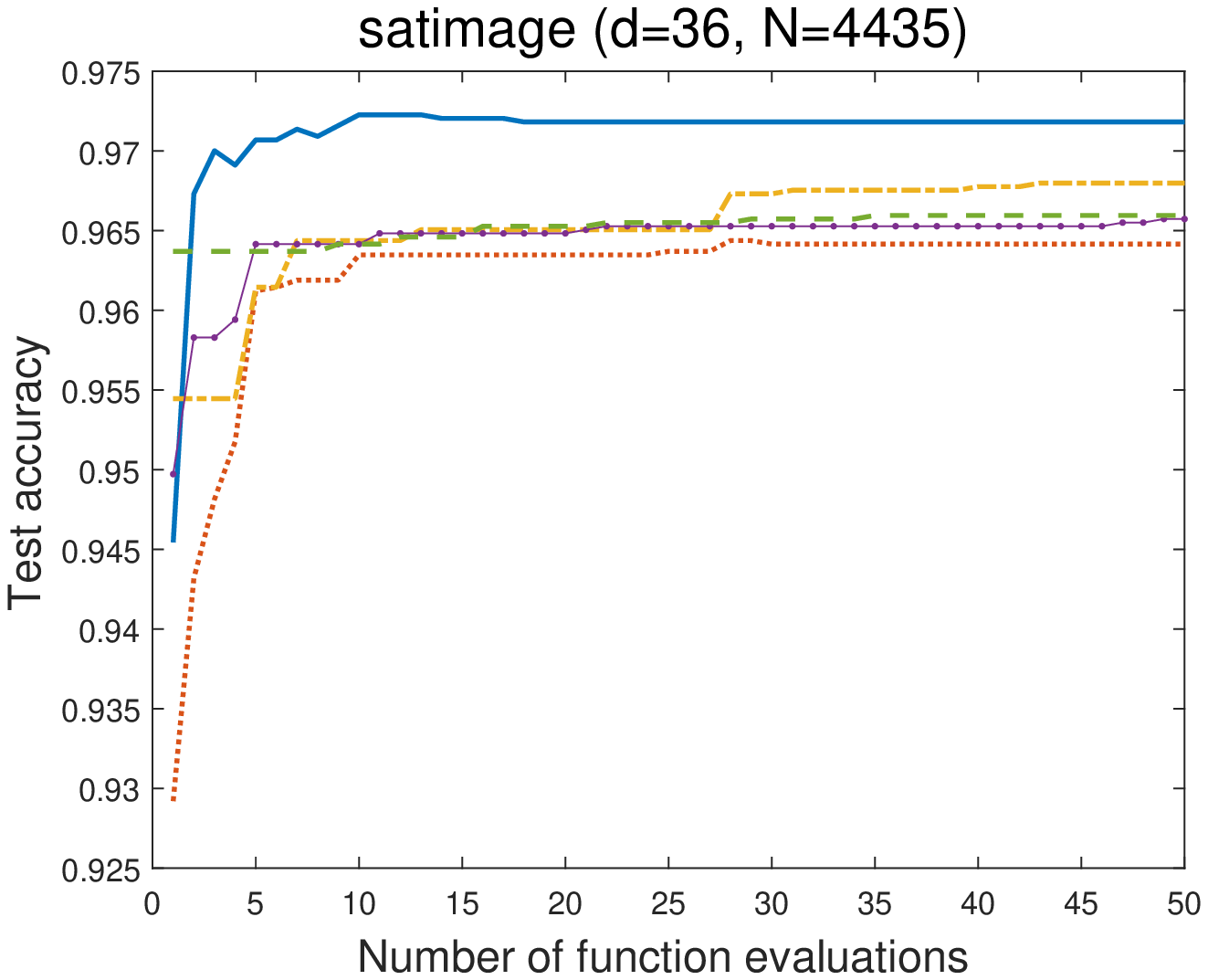}
 % \caption{satimage (d=36, N=4435)}
  \label{fig:sfig6}
\end{subfigure} %

\caption{Comparison of DFO-TR, random search and BO algorithms on tuning RBF-kernel SVM hyperparameters.}
\label{fig svm}
\end{figure}
%
%

%*********
% Section
%*********
\section{Conclusion}\label{fourth} 
In this work, we demonstrate that model-based derivative free optimization is  a better alternative to Bayesian optimization for some black-box optimization tasks arising in machine learning. We rely on an existing convergent stochastic trust region framework to provide theoretical foundation for the chosen algorithm, and we demonstrate the efficiency of a practical implementation of DFO-TR for optimizing AUC function over the set of linear classifiers, hyperparameter tuning, and on other benchmark problems.

%Our experimental results show that, in terms of the average value of AUC, minimizing the pairwise hinge loss performs quite competitive to directly maximize AUC via \textit{DFO-TR}. However, in terms of computational efforts, minimizing the pairwise hinge loss needs more function evaluations. Furthermore, the other advantage of  \textit{DFO-TR} is that it does not compute any gradient, which has the computational complexity of $\mathcal{O}(N \log N)$, the same as computing the function values.
%
%Moreover, we proposed an optimization algorithm based on \textit{DFO-TR} and proximal quasi-Newton algorithm, in order to tune the parameters of the cost sensitive logistic regression, to build a linear classifier with high ranking performance. We guaranteed the performance of this algorithm under some practical assumptions. The computational results show that in 10 percent of our data sets, optimizing the parameters of the cost sensitive logistic regression improves the ranking quality of the classifier. 

\newpage
\nocite{langley00}
\bibliography{references}
\bibliographystyle{icml2017}

\onecolumn
\vskip 1.5in
\textbf{\Large{Appendix}} 

In this section, we provide some complementary results. Table \ref{t1_2} compares the performance of the linear classifier obtained by directly maximizing AUC via DFO-TR, versus approximately maximizing AUC by utilizing gradient descent approach to minimize the pair-wise hinge loss. We utilized the same strategy of using data sets as is done in \S\ref{subsub}. Therefore, the AUC value in Table \ref{t1_2} is the average AUC value of 20 different runs. As we can see, compared to minimizing pair-wise hinge loss, DFO-TR achieves competitive AUC value in less number of function evaluations.

Figures \ref{figap_1} and \ref{figap_2} support the computational results in \S\ref{BO} and illustrate the per iteration behavior of each method. We can see that DFO improves the objective values faster than the Bayesian optimization algorithms.
%]

\begin{table} [H]
%\small
 \centering
 \captionsetup{justification=centering}
  \caption{\normalsize{Comparing deterministic DFO-TR vs. pairwise hinge loss minimization}}
    \label{t1_2}
     \vskip 0.1in
  \begin{tabular}{ |c|cc|cc|}
\hline
\multirow{2}{*}{\textbf{Algorithm}} 
& \multicolumn{2}{c|}{\textbf{fourclass}}  &\multicolumn{2}{c|}{\textbf{svmguide1}}\\   
& {AUC}&{num. fevals}&{AUC}&{num. fevals}\\  \hline
Pair-Wise Hinge&0.836226$\pm$ 0.000923&480&0.989319 $\pm$ 0.000008&334\\
DFO-TR&0.836311$\pm$ 0.000921&254&0.989132$\pm$ 0.000007&254\\ \hline
\multirow{2}{*}{\textbf{Algorithm}} 
& \multicolumn{2}{c|}{\textbf{diabetes}}  &\multicolumn{2}{c|}{\textbf{shuttle}}\\   
& {AUC}&{num. fevals}&{AUC}&{num. fevals}\\  \hline
Pair-Wise Hinge&0.830852$\pm$0.001015&348&0.988625$\pm$0.000021&266\\ 
DFO-TR&0.830402$\pm$0.00106&254&0.987531$\pm$0.000035&254\\ \hline
\multirow{2}{*}{\textbf{Algorithm}} 
& \multicolumn{2}{c|}{\textbf{vowel}}  &\multicolumn{2}{c|}{\textbf{magic04}}\\   
& {AUC}&{num. fevals}&{AUC}&{num. fevals}\\  \hline
Pair-Wise Hinge&0.975586$\pm$0.000396&348&0.843085 $\pm$ 0.000208&417\\ 
DFO-TR&0.973785$\pm$0.000506&254&0.843511$\pm$0.000213&254\\ \hline
\multirow{2}{*}{\textbf{Algorithm}} 
& \multicolumn{2}{c|}{\textbf{letter}}  &\multicolumn{2}{c|}{\textbf{segment}}\\   
& {AUC}&{num. fevals}&{AUC}&{num. fevals}\\  \hline
Pair-Wise Hinge&0.986699$\pm$0.000037&517&0.993134 $\pm$ 0.000023&753\\
DFO-TR&0.985119$\pm$0.000042&254&0.99567$\pm$0.00071&254\\ \hline
\multirow{2}{*}{\textbf{Algorithm}} 
& \multicolumn{2}{c|}{\textbf{ijcnn1}}  &\multicolumn{2}{c|}{\textbf{svmguide3}}\\   
& {AUC}&{num. fevals}&{AUC}&{num. fevals}\\  \hline
Pair-Wise Hinge&0.930685$\pm$0.000204&413&0.793116$\pm$0.001284&368\\
DFO-TR&0.910897$\pm$ 0.000264&254&0.775246$\pm$ 0.002083&254\\ \hline
\multirow{2}{*}{\textbf{Algorithm}} 
& \multicolumn{2}{c|}{\textbf{german}}  &\multicolumn{2}{c|}{\textbf{satimage}}\\
& {AUC}&{num. fevals}&{AUC}&{num. fevals}\\  \hline
Pair-Wise Hinge&0.792402$\pm$ 0.000795&421&0.769505$\pm$0.000253&763\\
DFO-TR&0.791048$\pm$0.000846&254&0.757554$\pm$0.000236&254\\ \hline
\end{tabular}
\centering
\end{table}

\begin{figure*} [t]
\begin{subfigure} 
  \centering
  \includegraphics[height=0.38\textwidth,width=0.5\linewidth]{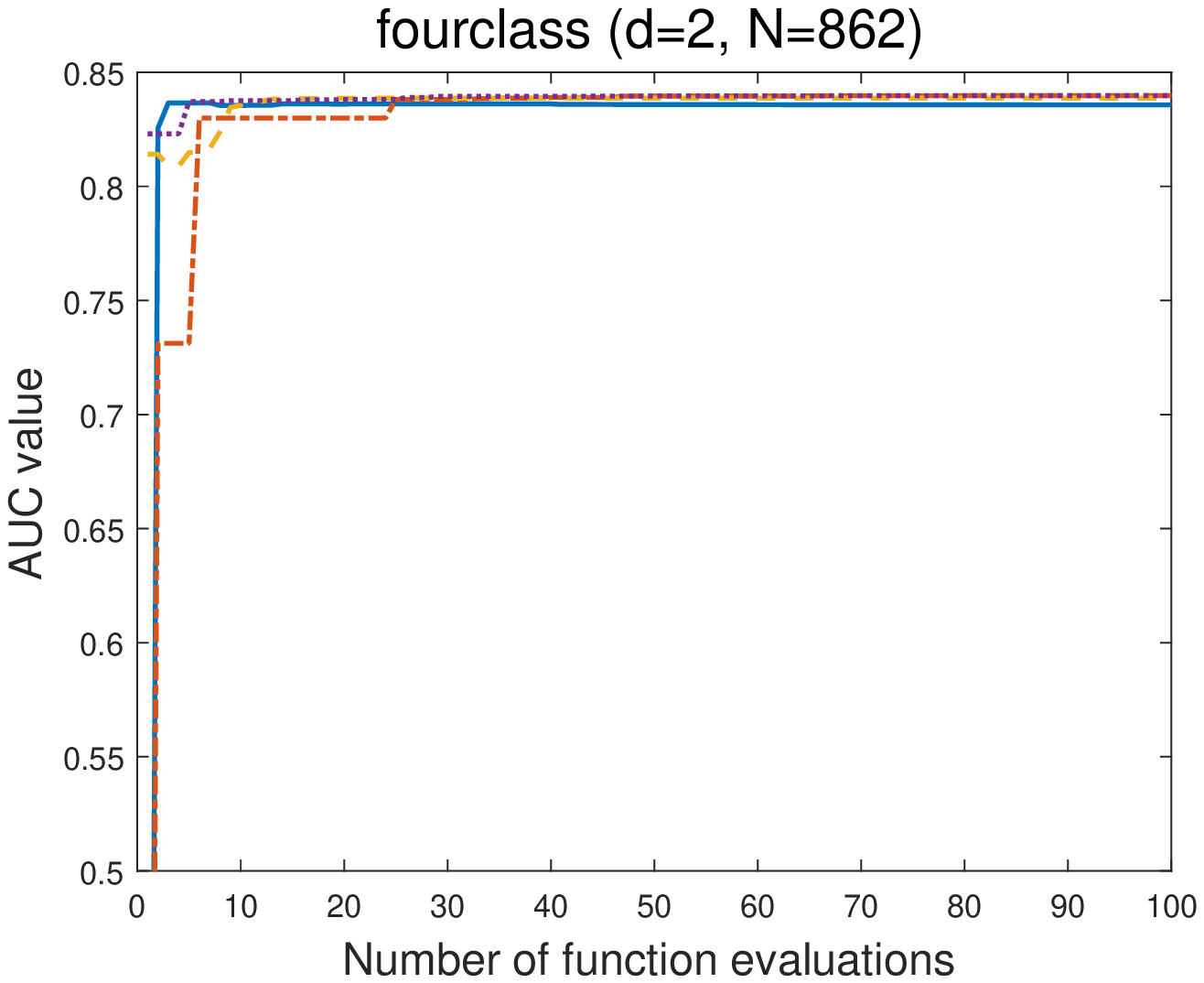}
  %\caption{fourclass(d=2, N=862)}
  \label{fig:sfig1}
\end{subfigure}%
\begin{subfigure}
  \centering
  \includegraphics[height=0.38\textwidth,width=0.5\linewidth]{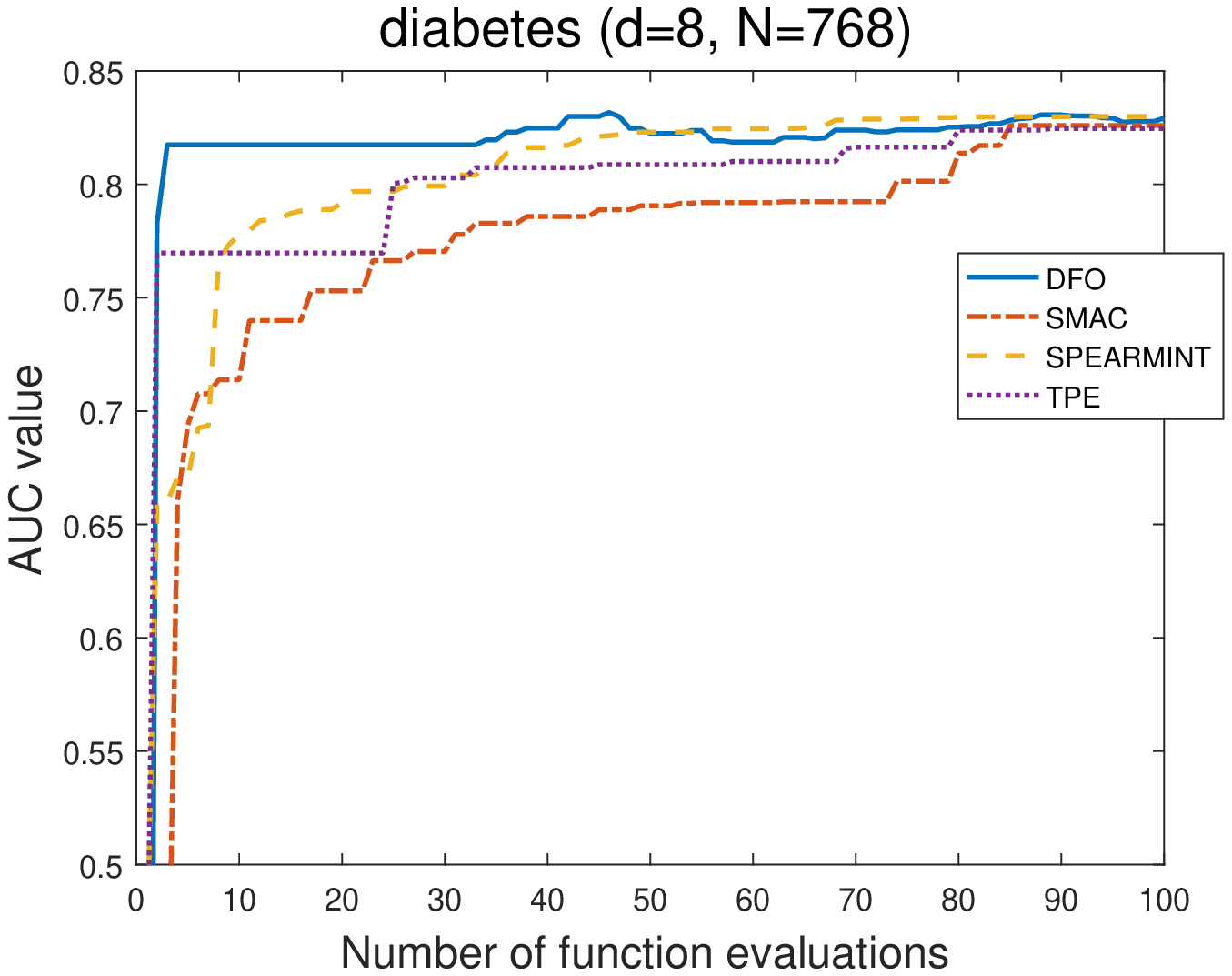}
 % \caption{diabetes(d=8, N=768)}
  \label{fig:sfig2}
\end{subfigure}%
\begin{subfigure}
  \centering
  \includegraphics[height=0.38\textwidth,width=0.5\linewidth]{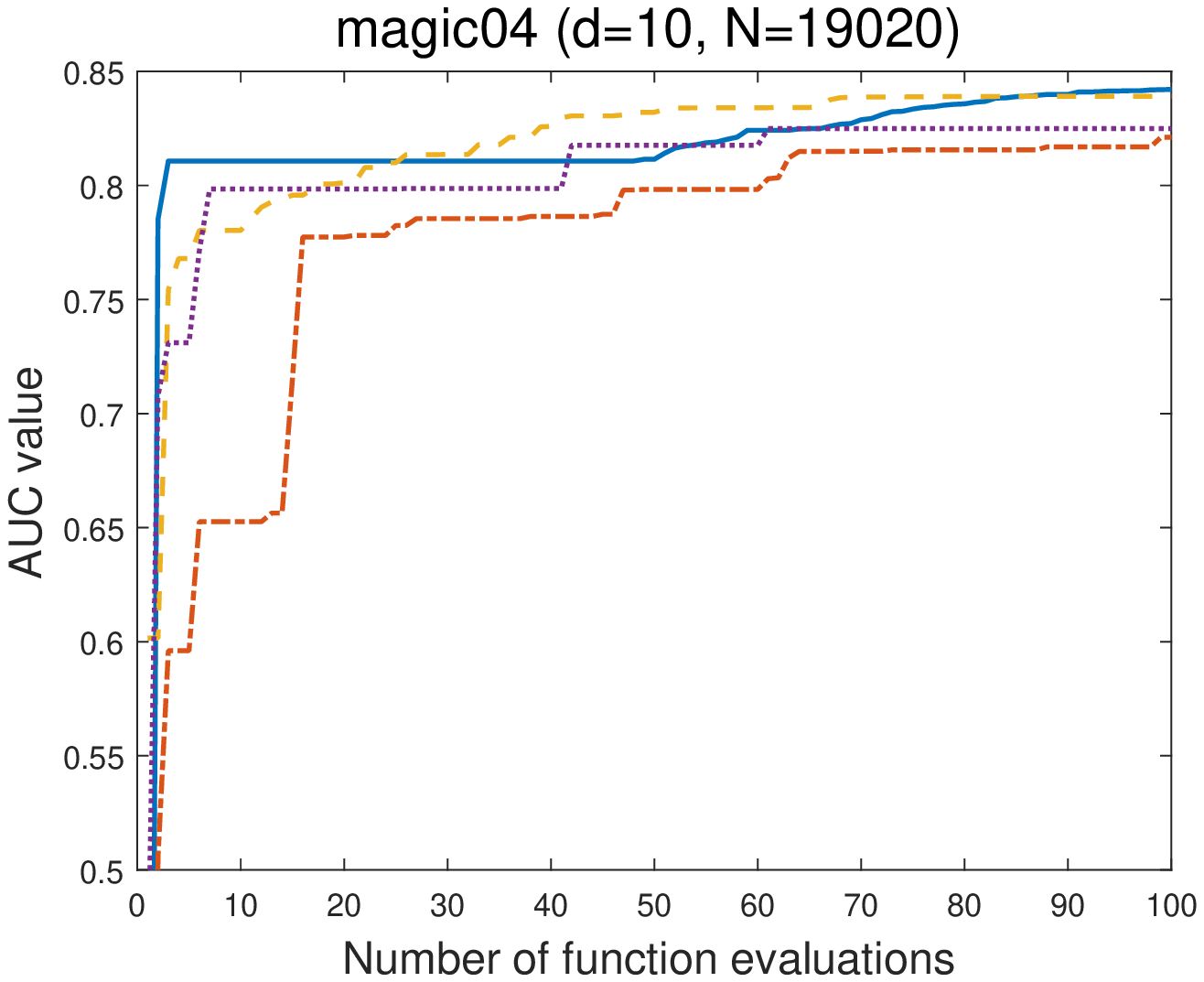}
  %\caption{magic04(d=10, N=19020)}
  \label{fig:sfig1}
\end{subfigure}%
\begin{subfigure}
  \centering
  \includegraphics[height=0.38\textwidth,width=0.5\linewidth]{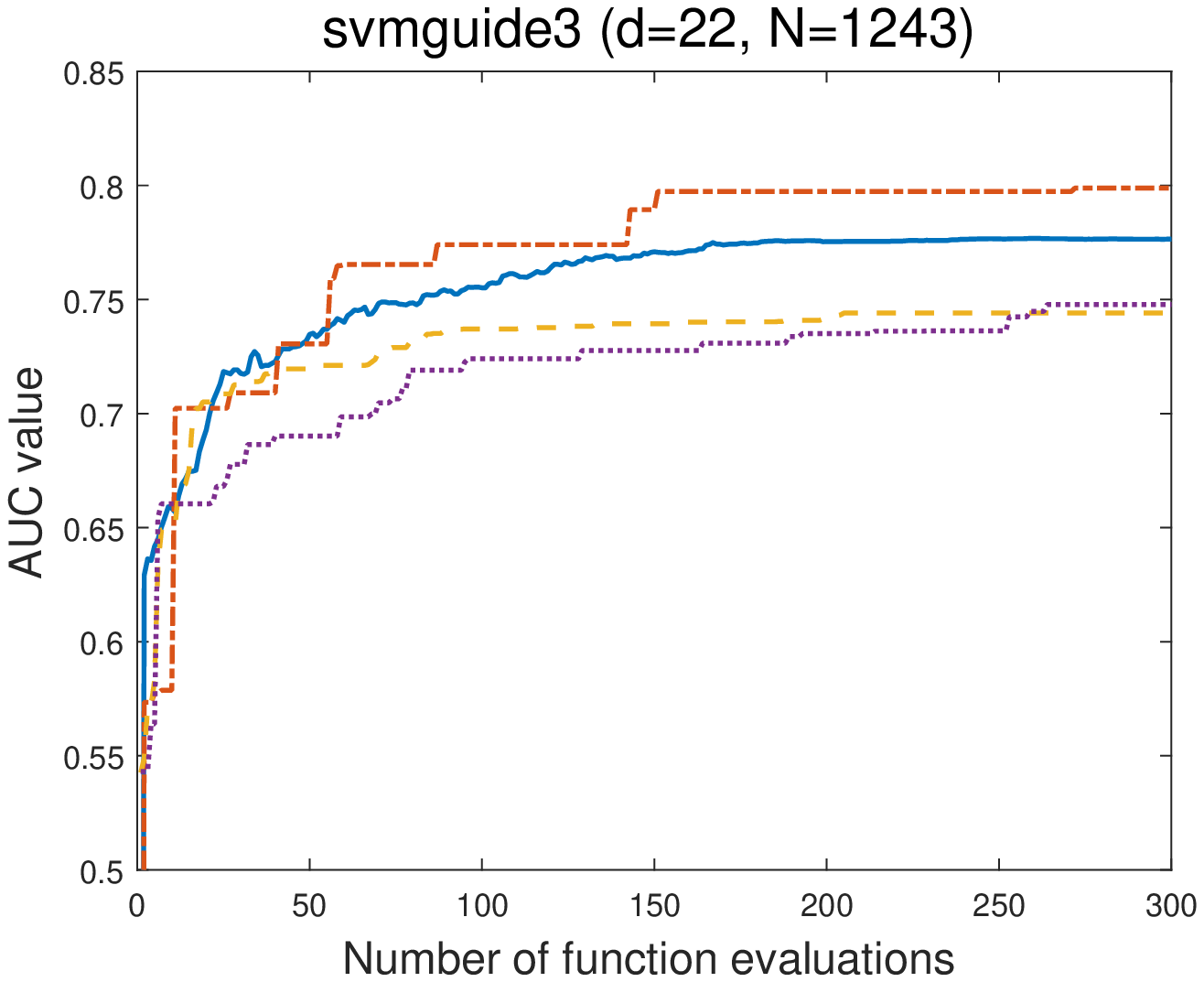}
  %\caption{svmguide3(d=22, N=1243)}
  \label{fig:sfig2}
\end{subfigure}%
\begin{subfigure}
  \centering
  \includegraphics[height=0.38\textwidth,width=0.5\linewidth]{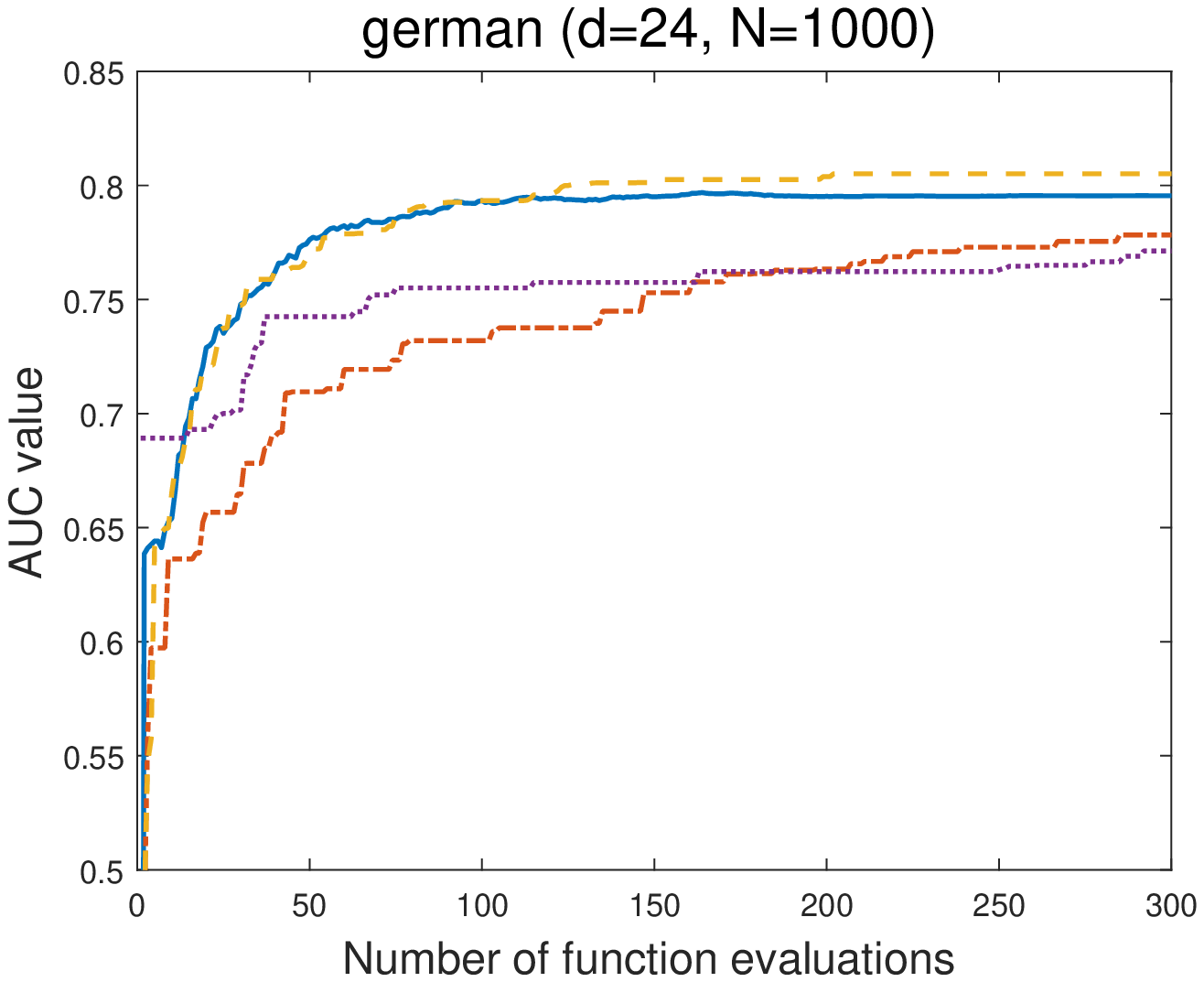}
 % \caption{german(d=24, N=1000)}
  \label{fig:sfig5}
\end{subfigure}%
\begin{subfigure}
  \centering
  \includegraphics[height=0.38\textwidth,width=0.5\linewidth]{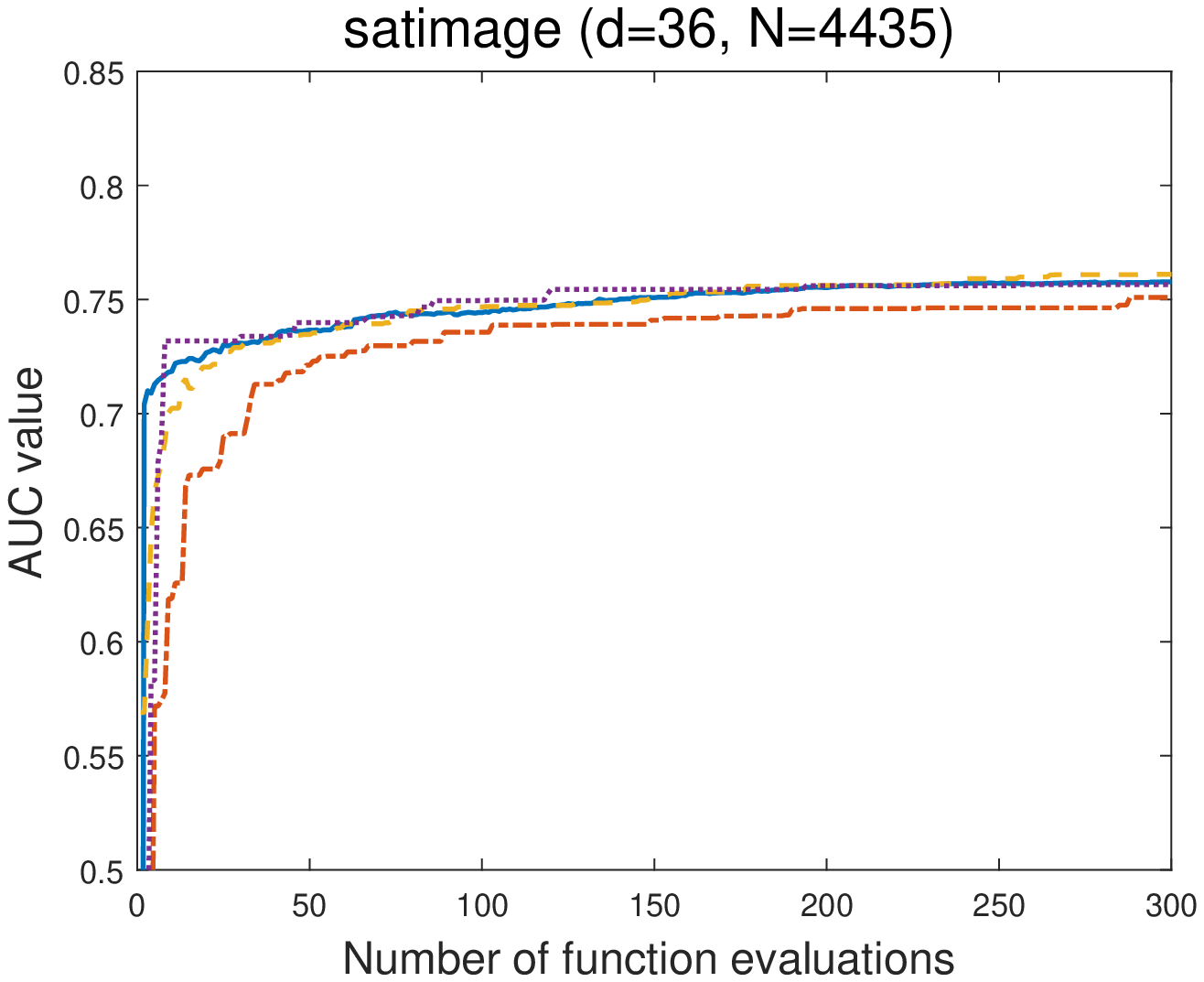}
 % \caption{satimage(d=36, N=4435)}
  \label{fig:sfig6}
\end{subfigure}%
\caption{\normalsize{DFO-TR vs. BO algorithms (First Part).}}
\label{figap_1}
\end{figure*}

\begin{figure*}[t]
\begin{subfigure}
  \centering
  \includegraphics[height=0.38\textwidth,width=0.5\linewidth]{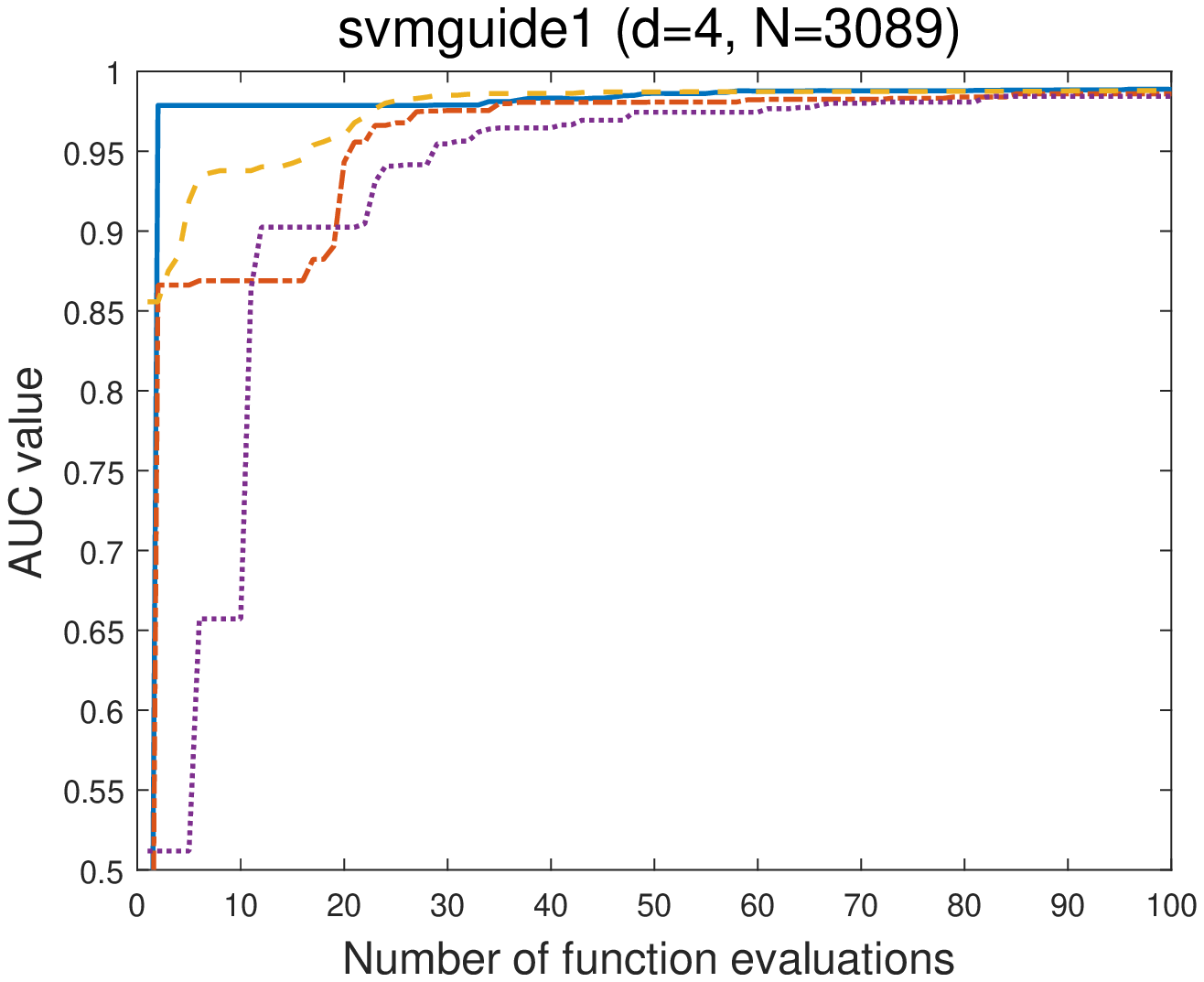}
  %\caption{svmguide1(d=4, N=3089)}
  \label{fig:sfig1}
\end{subfigure}%
\begin{subfigure}
  \centering
  \includegraphics[height=0.38\textwidth,width=0.5\linewidth]{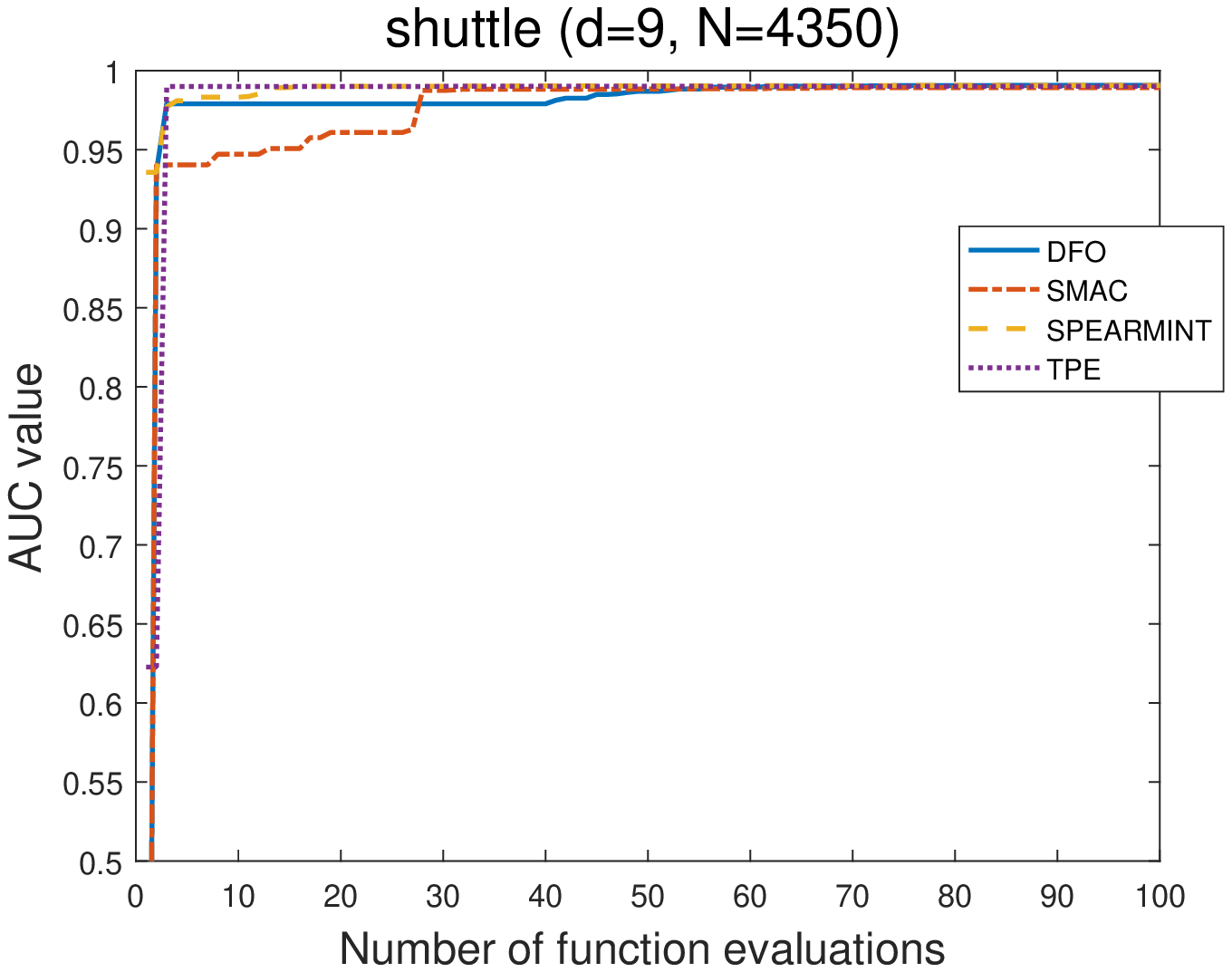}
 % \caption{shuttle(d=9, N=43500)}
  \label{fig:sfig2}
\end{subfigure} %
\begin{subfigure}
  \centering
  \includegraphics[height=0.38\textwidth,width=0.5\linewidth]{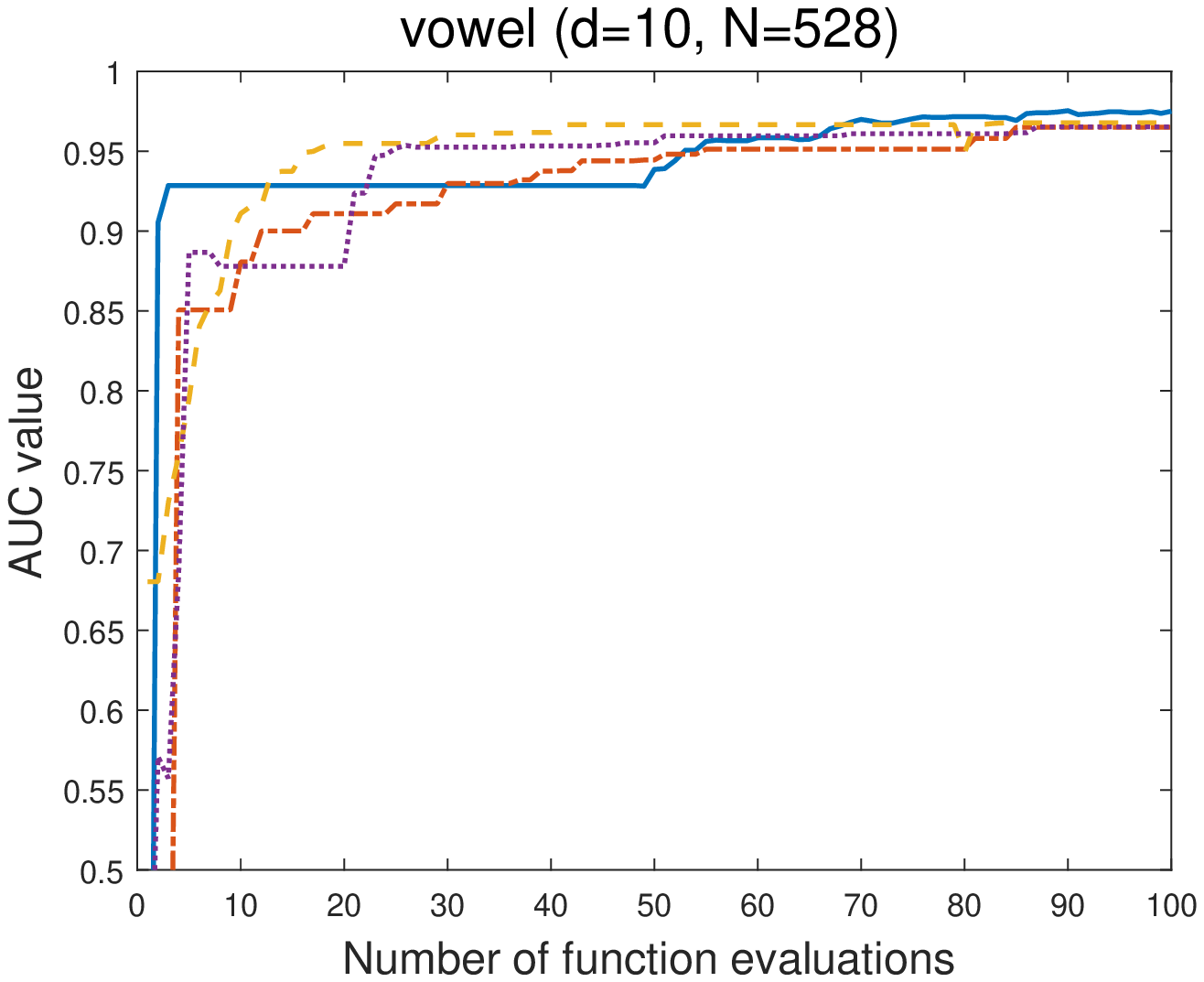}
 % \caption{vowel(d=10, N=528)}
  \label{fig:sfig3}
\end{subfigure}%
\begin{subfigure}
  \centering
  \includegraphics[height=0.38\textwidth,width=0.5\linewidth]{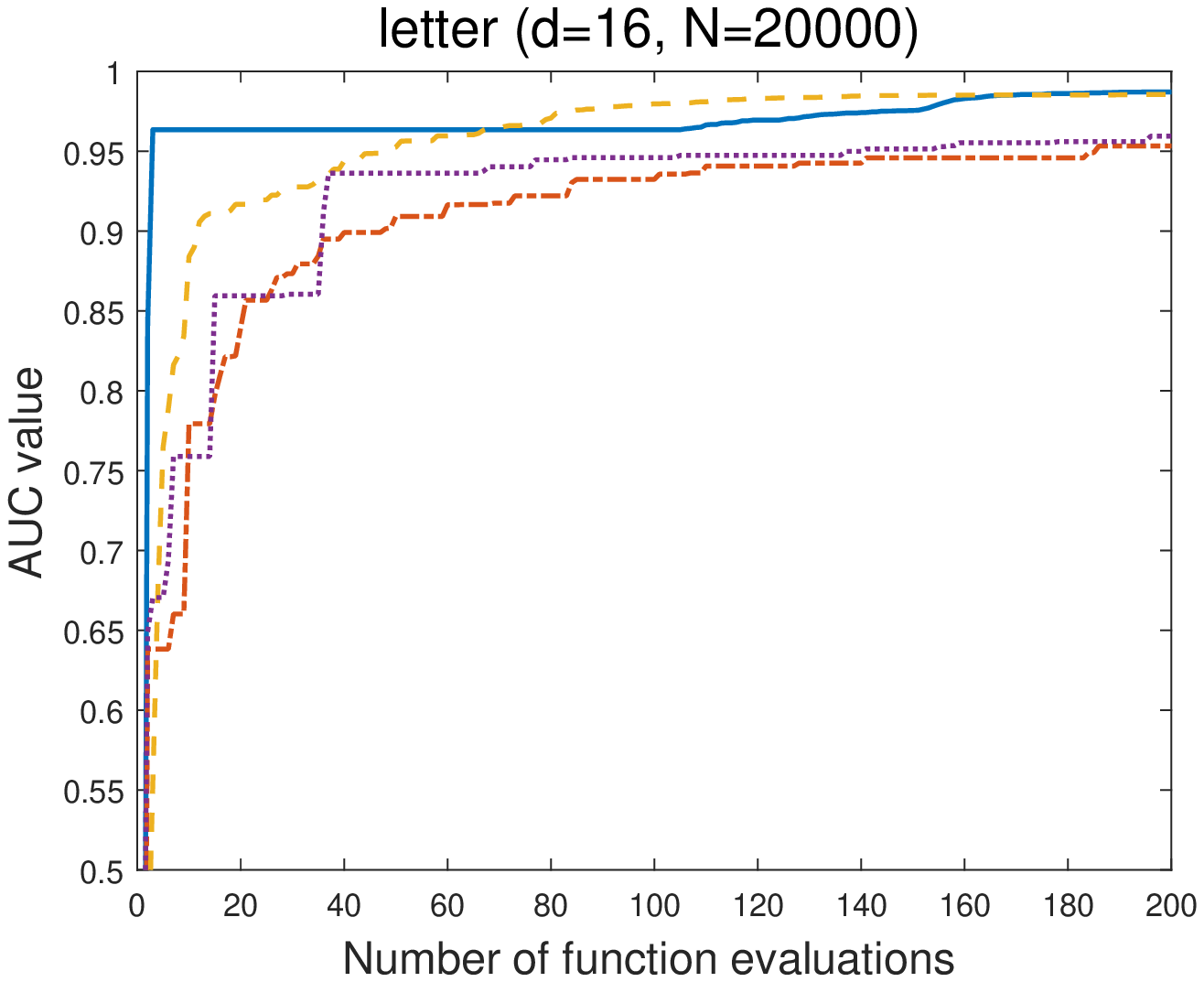}
%  \caption{letter(d=16, N=20000)}
  \label{fig:sfig4}
\end{subfigure} %
\begin{subfigure}
  \centering
  \includegraphics[height=0.38\textwidth,width=0.5\linewidth]{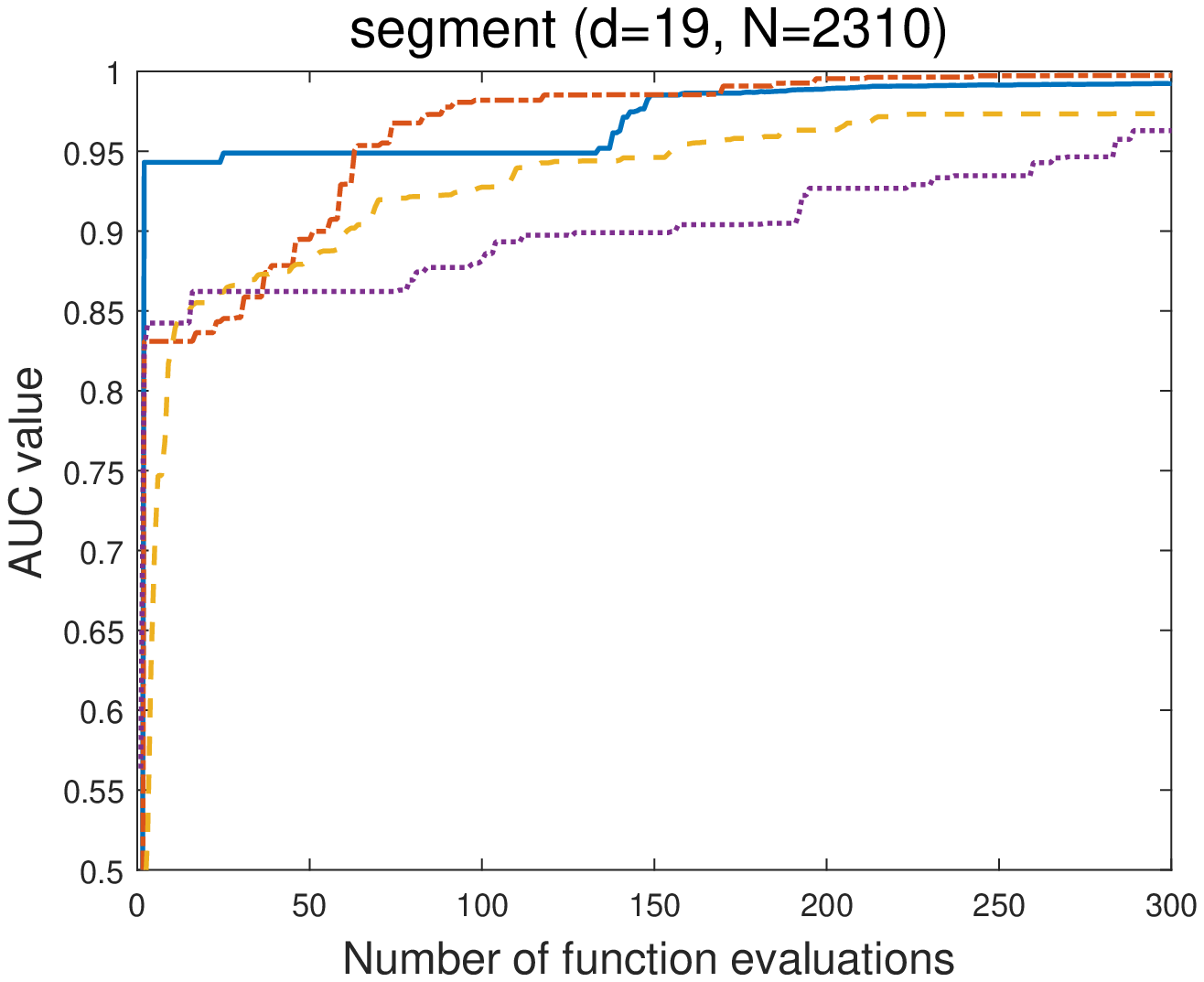}
%  \caption{segment(d=19, N=2310)}
  \label{fig:sfig5}
\end{subfigure}%
\begin{subfigure}
  \centering
  \includegraphics[height=0.38\textwidth,width=0.5\linewidth]{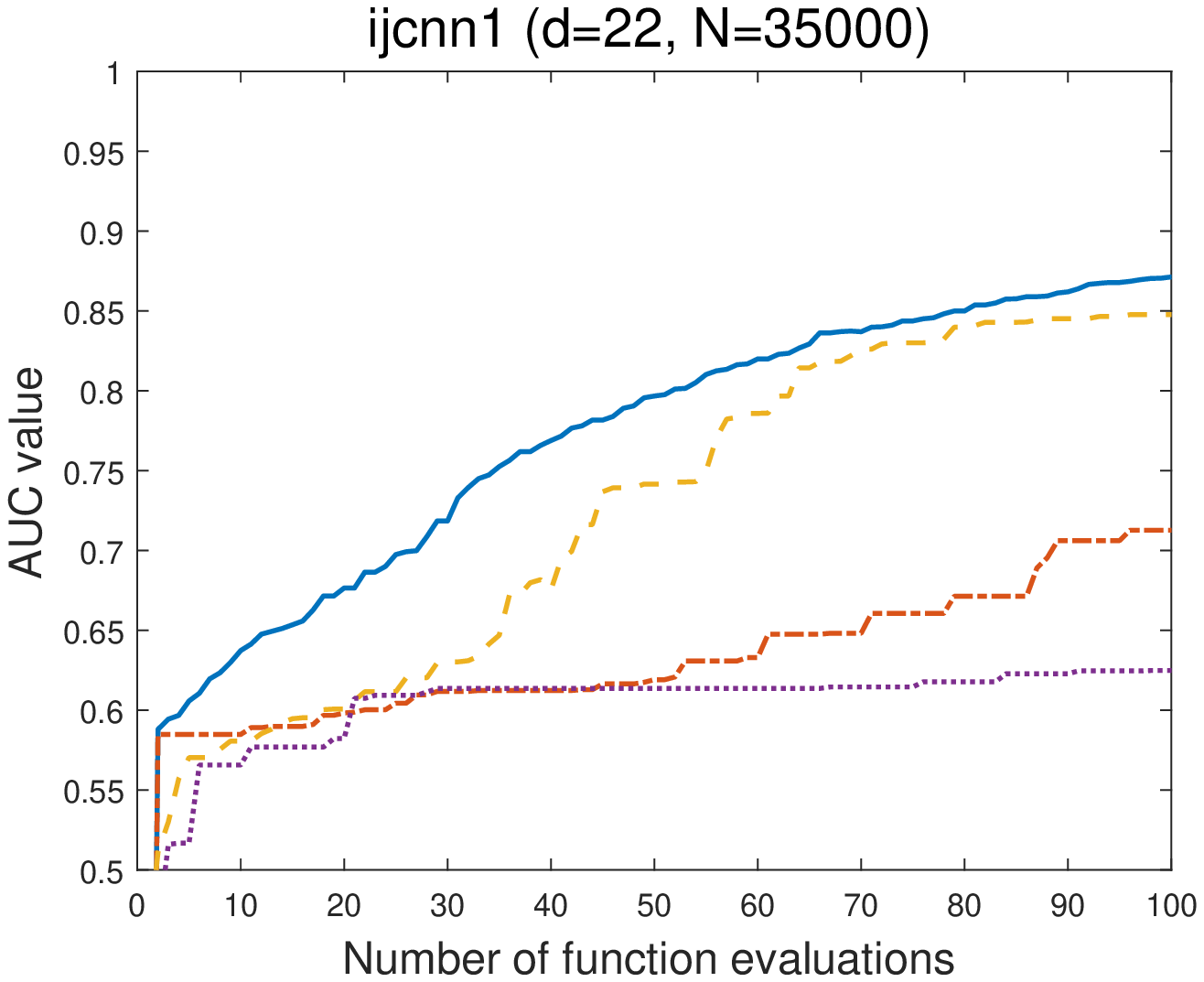}
 % \caption{ijcnn1(d=22, N=35000)}
  \label{fig:sfig6}
\end{subfigure} %
\caption{DFO-TR vs. BO algorithms (Second Part).}
\label{figap_2}
\end{figure*}

\end{document}